\documentclass[twoside]{article}

\usepackage[accepted]{aistats2021}

\usepackage[utf8]{inputenc} % allow utf-8 input
\usepackage[T1]{fontenc}    % use 8-bit T1 fonts
\usepackage{hyperref}       % hyperlinks
\usepackage{url}            % simple URL typesetting
\usepackage{booktabs}       % professional-quality tables
\usepackage{amsfonts}       % blackboard math symbols
\usepackage{nicefrac}       % compact symbols for 1/2, etc.
\usepackage{microtype}      % microtypography
\usepackage{hyperref}
\usepackage{url}
\usepackage{amssymb} % Adds new symbols to be used in math
\usepackage{amsthm} 
\usepackage{amsmath}
\usepackage{epstopdf} 
\usepackage{float}
\usepackage{mathtools}
\usepackage{natbib}
\usepackage{xcolor}
\usepackage{subcaption}
\usepackage{adjustbox}
\usepackage{stackengine}
\usepackage{multirow}
\usepackage{dirtytalk}

\usepackage[ruled, lined, longend]{algorithm2e}

\definecolor{bluecustom}{rgb}{0.12156862745098039,0.4666666666666667,0.7058823529411765}
\definecolor{orangecustom}{rgb}{1.0,0.4980392156862745,0.054901960784313725}
\definecolor{greencustom}{rgb}{0.17254901960784313,0.6274509803921569, 0.17254901960784313}
\definecolor{redcustom}{rgb}{0.8392156862745098,0.15294117647058825, 0.1568627450980392}

\newcommand{\cD}{\mathcal{D}}

\newcommand{\cH}{\mathcal{H}}

\newcommand{\cL}{\mathcal{L}}

\newcommand{\cN}{\mathcal{N}}

\newcommand{\cX}{\mathcal{X}}
\newcommand{\cY}{\mathcal{Y}}

\newcommand{\w}{\omega}
\newcommand{\tl}{\tilde}

\newcommand{\algoname}{PCA-OGD\xspace}

\newcommand\cut[1]{}

\newcommand{\squishlist}{
	\begin{list}{$\bullet$}
		{ \setlength{\itemsep}{0pt}      \setlength{\parsep}{3pt}
			\setlength{\topsep}{3pt}       \setlength{\partopsep}{0pt}
			\setlength{\leftmargin}{1.5em} \setlength{\labelwidth}{1em}
			\setlength{\labelsep}{0.5em} } }

	\newcommand{\squishlisttwo}{
		\begin{list}{$\bullet$}
			{ \setlength{\itemsep}{0pt}    \setlength{\parsep}{0pt}
				\setlength{\topsep}{0pt}     \setlength{\partopsep}{0pt}
				\setlength{\leftmargin}{2em} \setlength{\labelwidth}{1.5em}
				\setlength{\labelsep}{0.5em} } }

		\newcommand{\squishend}{
		\end{list}  }

	{}
	\newtheorem{thm}{Theorem}{}
	{}
	\newtheorem{corr}{Corollary}{}
	\newtheorem{lemma}{Lemma}{}
	{}
	\newtheorem{defn}{Definition}{}
	\newtheorem{remark}{Remark}{}

	\DeclareMathOperator*{\argmin}{arg\,min}

	\newcommand{\myvec}[1]{\mbox{$\mathbf{#1}$}}

	\newcommand{\vg}{\mbox{$\myvec{g}$}}

	\newcommand{\vecu}{\mbox{$\myvec{u}$}}
	\newcommand{\vv}{\mbox{$\myvec{v}$}}

	\newcommand{\vx}{\mbox{$\myvec{x}$}}

	\newcommand{\be}{\begin{equation}}
	\newcommand{\ee}{\end{equation}}
	\newcommand{\bea}{\begin{eqnarray}}
	\newcommand{\eea}{\end{eqnarray}}
	\newcommand{\beaa}{\begin{eqnarray*}}
	\newcommand{\eeaa}{\end{eqnarray*}}

	\newcommand{\norm}[1]{\left\lVert#1\right\rVert}

\newcommand{\reals}{\mbox{$\mathbb{R}$}}

\newcommand{\naturals}{\mbox{$\mathbb{N}$}}

\usepackage{accents}

	\newcommand{\cT}{\mbox{$\mathcal{T}$}}
	\newcommand{\cS}{\mbox{$\mathcal{S}$}}


\usepackage{lipsum}

\newcommand\blfootnote[1]{%
  \begingroup
  \renewcommand\thefootnote{}\footnote{#1}%
  \addtocounter{footnote}{-1}%
  \endgroup
}

\begin{document}

\newcommand\mycommfont[1]{\small\ttfamily\textcolor{blue}{#1}}
\SetCommentSty{mycommfont}


\twocolumn[

\aistatstitle{A Theoretical Analysis of Catastrophic Forgetting through the NTK Overlap Matrix}

\aistatsauthor{ Thang Doan \footnotemark[1] \footnotemark[2]
\And Mehdi Bennani \footnotemark[3]
\And  Bogdan Mazoure  \footnotemark[1] \footnotemark[2]
\AND Guillaume Rabusseau  \footnotemark[2] \footnotemark[4]
\And Pierre Alquier \footnotemark[5] } \vspace{1cm}
% \aistatsaddress{ \footnotemark[1] McGill University \footnotemark[2] Mila  \footnotemark[3] Aquemia \footnotemark[4]  Université de Montreal \footnotemark[5] RIKEN AIP } 
]

\begin{abstract}
    Continual learning (CL) is a setting in which an agent has to learn from an incoming stream of data during its entire lifetime. Although major advances have been made in the field, one recurring problem which remains unsolved is that of \emph{Catastrophic Forgetting} (CF). While the issue has been extensively studied empirically, little attention has been paid from a theoretical angle. In this paper, we show that the impact of CF increases as two tasks increasingly align. We introduce a measure of task similarity called the \emph{NTK overlap matrix} which is at the core of CF. We analyze common projected gradient algorithms and demonstrate how they mitigate forgetting. Then, we propose a variant of Orthogonal Gradient Descent (OGD) which leverages structure of the data through Principal Component Analysis (PCA). Experiments support our theoretical findings and show how our method can help reduce CF on classical CL datasets. 
\end{abstract}

% % \footnotetext{\footnotemark[1] McGill University    \footnotemark[2] Mila \footnotemark[3] Aqemia \footnotemark[4] Université de Montréal \footnotemark[5] RIKEN AIP}
% \item dire que les échantillons c'est 250 samples 
% \item on ne fait pas l'update orthogonalement aux CNN 
% \end{itemize}

\section{Introduction}

Continual learning (CL) or lifelong learning   \citep{thrun1995lifelong,chen2018lifelong} has been one of the most
important milestone on the path to building artificial general intelligence \citep{agi}.
This setting refers to learning from an incoming stream of data, as well as leveraging previous knowledge for future tasks (through forward-backward transfer \citep{gem}).
While the topic has seen increasing interest in the past years~\citep{de2019defying_forgetting,parisi2019icarl} and a
number of sohpisticated methods have been  developed~\citep{ewc,gem,agem,aljundi2019gradient}, a yet unsolved central
challenge remains: Catastrophic Forgetting (CF) \citep{goodfellow2013empirical,mccloskey1989catastrophic}.

CF occurs when past solutions degrade while learning from new incoming tasks according to non-stationary distributions.
Previous work either investigated this phenomenon empirically at different granularity levels (task level
\citep{nguyen2019forgetting}, neural network representations level \citep{ramasesh2020anatomy} ),
or proposed a quantitative metric \citep{farquhar2018towards,kemker2017measuring,nguyen2020explaining}.

Despite the vast set of existing works on CF, there is still few theoretical works studying this major topic. Recently, \cite{bennani2020generalisation} propose a framework to study Continual Learning in the NTK regime then derive generalization guarantees of CL under the Neural Tangent 
Kernel \cite[NTK]{jacot2018neural} for Orthogonal Gradient Descent \cite[OGD]{farajtabar2020orthogonal}.
Following on this work, we propose a theoretical analysis of Catastrophic Forgetting for a family of
projection algorithms including OGD, GEM \citep{gem}.
Our contributions can be summarized as follows:
\blfootnote{$^1$McGill University    $^2$Mila $^3$Aqemia $^4$Université de Montréal $^5$RIKEN AIP \\ 
corresponding author: thang.doan@mail.mcgill.ca}
\begin{itemize}
    \item We provide a general definition of Catastrophic Forgetting, and examine the special case of CF  under the Neural Tangent Kernel (NTK) regime. Our definition leverages the similarity between the source and target task. 
    \item We derive the expression of the forgetting error for a family of orthogonal projection methods based on the \emph{NTK overlap matrix}. This matrix reduces to the angle between the source and target tasks and is a critical component responsible for the Catastrophic Forgetting.
    \item For these projection methods, we analyze their mechanisms to reduce Catastrophic Forgetting and how they 
    differ from each other. 
    \item We propose PCA-OGD, an extension 
    of OGD which mitigates the CF issue by compressing the relevant information into a reduced number of principal components.
    We show that our method is advantageous whenever the dataset has a dependence pattern between tasks.
\end{itemize}

\begin{figure}
    \centering
    \includegraphics[width=\linewidth]{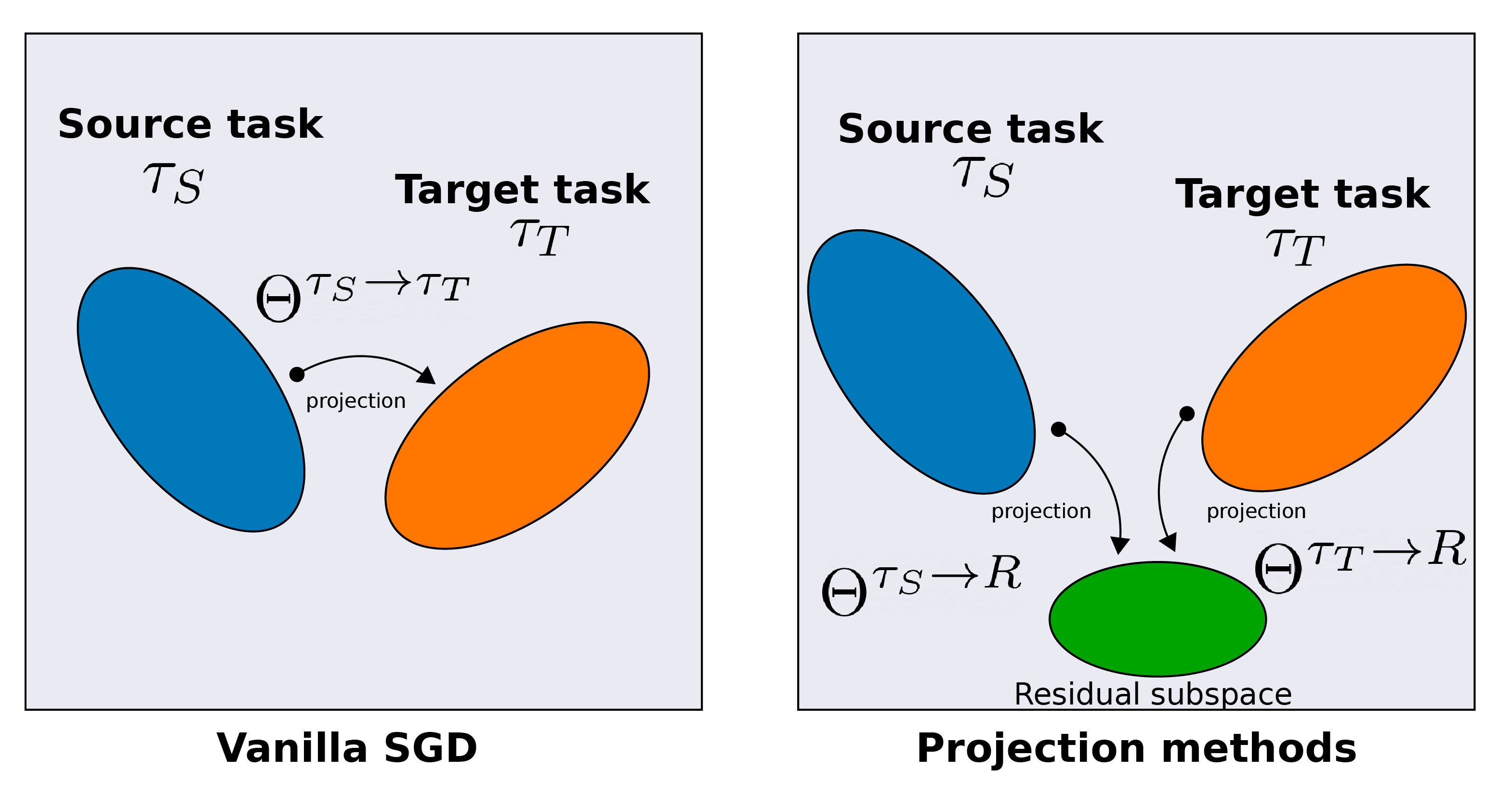}
    \caption{Unlike SGD, the projection methods reduce the forgetting by projecting the source and target tasks on a residual subspace.}
    \label{fig:my_label}
\end{figure}

\section{Related Works}

% Recently different methods have tried to tackle this problem either by or by....
% But there are still a lot to understand around CF since it is a non trivial problem. Recent papers trying to under that either studied deeply by visualization of the hidden layers or throught embedding or even measure catastrophic forgetting. While there are lots of empirical studies on this CF very few theoretical analysis on CF the only one is the regret
% agem gem ogd méthode de projection 
% regularization based methods 

Defying Catastrophic Forgetting \citep{mccloskey1989catastrophic} has always been an important challenge  for the
Continual Learning community.
Among different families of methods, we can cite the following: regularization-based methods \citep{ewc,zenke2017continual},
memory-based and projection methods \citep{agem,gem,farajtabar2020orthogonal} or parameters isolations \citep{mallya2018packnet,rosenfeld2018incremental}. In \citep{pan2020continual}, the authors propose a method to identify memorable example from the past that must be stored.
An exhaustive list can be found in \citep{de2019defying_forgetting}. 
Although theses methods achieve more or less success in combating Catastrophic Forgetting, its underlying theory remains unclear.
% there is still a gap 
% towards understanding fully that phenomenon. 

Recently, a lot of efforts has been put toward dissecting CF \citep{toneva2018empirical}.
While \cite{nguyen2019forgetting} empirically studied the impact of tasks similarity on the forgetting, \citep{ramasesh2020anatomy} analyzed this phenomenon at the neural network layers level. \citep{xie2020artificial} studied the designed \textit{artificial neural variability} and studied its impact on the forgetting. \cite{mirzadeh2020understanding} investigated how different training regimes affected the forgetting.
Other streams of works investigated different evaluation protocol and measure of CF \citep{farquhar2018towards,kemker2017measuring}.
That being said, there is still few theoretical work confirming empirical evidences of CF.

\cite{yin2020sola} provide an analysis of CL from a loss landscape perspective through a second-order Taylor approximation.
Recent advances towards understanding neural networks behavior \citep{jacot2018neural} has enabled a better understanding through Neural Tangent Kernel (NTK) \citep{du2018gradient,arora2019exact}.
Latest work and important milestone towards better theoretical understand of CL is from \cite{bennani2020generalisation}.
The authors provide a theoretical framework for CL under the NTK regime for the infinite memory case. Our work relaxes this constraint to the \emph{finite memory} case, which is more applicable in the empirical setting.

\section{Preliminaries}

\subsection{Notations}

% \mehdi{Suggestion from the first paper below :}
% We use bold-faced characters for vectors and matrices. 
We use $\norm{\cdot}_2$ to denote the Euclidian norm of a vector or
the spectral norm of a matrix. We use $\langle\cdot,
\cdot\rangle$ for the Euclidean dot product, and $\langle\cdot,\cdot\rangle_{\cH}$ the dot product in the Hilbert space $\cH$. We
index the task ID by $\tau$. Learnable parameter are denoted $\w$ and when indexing as $\w_{\tau}$ correspond to the training during task $\tau$. Moreover $\star$ represents the variable at the end of a given task, i.e $w_{\tau}^{*}$ represents the learned parameters at the end of task $\tau$. 

We denote $\naturals$ the set of natural numbers, $\reals$ the space of real numbers and $\naturals^{\star}$
for the set $\naturals\smallsetminus\{0\}$.

\subsection{Continual Learning}
Let $\cX$ be some feature space of interest (we take $\cX=\mathbb{R}^p$), and let $\cY$ be the space of labels (we let $\cY=\mathbb{R}$, but $\cY=\Delta^K$ can be used for classification\footnote{$\Delta^K$ denotes the vertices of the probability simplex of dimension $K$}).
In CL, we receive a stream of supervised learning tasks $\cT_{\tau}, \tau \in [T]$ where $\cT_{\tau}=\{x^{\tau}_{j},y^{\tau}_{j}\}_{j=1}^{n_{\tau}}$, with $T \in \mathbb{N}^{*}$.
While $X^{\tau} \in \mathbb{R}^{n_{\tau} \times p}$ ($p$ being the number of features) represents the dataset of task $\cT_{\tau}$ and $x^{\tau}_{j}$, $j=1,...,n_{\tau} \in \cX$ is a sample with its corresponding label $y^{\tau}_{j} \in \cY$. The goal is to learn a predictor $f_{\w}:\cX \times \cT \rightarrow \cY$ with $\w \in \mathbb{R}^{p}$ the parameters that will perform a prediction as accurate as possible. In the framework of CL, one cannot recover samples from previous tasks unless storing them in a memory buffer \citep{gem,parisi2019icarl}.

% \mehdi{Would also mention that we cannot recover data from the previous tasks } 

\subsection{NTK framework for Continual Learning}

\cite{lee2019wide} recently proved that under the NTK regime neural networks evolve as a linear model:
\begin{align*}
    f_{\tau}^{*}(x)=f_{\tau-1}^{*}(x)+\langle \nabla_{\w}f_{0}(x) , \w_{\tau}^{*}-\w_{\tau-1}^{*} \rangle
\end{align*}
with $\w_{\tau}^{\star}$ being the final weight after training on task $\tau$. The latter formulation implies the feature maps $\phi(x)=\nabla_{\w}f_{0}(x) \in \mathbb{R}^{1 \times p}$ is constant over time. Under that framework, 
\citep{bennani2020generalisation} show that CL models can be expressed as a recursive kernel regression and prove generalization and performance guarantee of OGD under infinite memory setting. 
We build up on this theoretical framework to study CF and quantify how the tasks similarity imply forgetting through 
the lens of eigenvalues and singular values decomposition (PCA and SVD).

\section{Analysis of Catastrophic Forgetting in finite memory}
In this section, we propose a general definition of Catastrophic Forgetting (CF).
Casted in the NTK framework, this definition allows to understand what are the main sources of CF.  Namely, CF is likely to occur when two tasks align significantly.
Finally, we investigate CF properties for the vanilla case (SGD) and projection based methods such as OGD and a variant of GEM. We then introduce a new algorithm called \algoname, an extension of OGD which reduces CF .

\subsection{A definition of Catastrophic Forgetting under the NTK regime}
A natural quantity to characterize CF is the change in predictions for the same input between a source task $\tau_S$ and target task $\tau_T$.

\begin{defn} [Drift] \

Let $\tau_S$ (respectively $\tau_T$) be the source task (respectively target task), $\cD_{\tau_S}$ the source test set, the CF of task $\tau_S$ after training on all the subsequent tasks up to the target task $\tau_T$ is defined as:
\begin{align}
    \delta^{\tau_S \rightarrow \tau_T}(X^{\tau_S})= \Bigl(f^{\star}_{\tau_T}(x)-f^{\star}_{\tau_S}(x)\Bigr)_{(x,y)\in \cD_{\tau_S}}
\end{align}
\end{defn}
Note that $ \delta^{\tau_S \rightarrow \tau_T}(X^{\tau_S})$ is a vector in $\mathbb{R}^{n_{\tau_S}}$ that contains 
the changes of predictions for any input $x$ in the task $\tau_S$.
In the case of classification, we take the
$k$-output of $f_{\tau}^{\star}$ such that $y_k=1$. In order to quantify the overall forgetting on this task, we use the squared norm of this vector.
% \mehdi{I thought it was backward positive transfer, forgetting (backard negative) or transfer (forward) }
% \mehdi{Motivate the definition stating that backward positive transfer is a desirable property (cite)}

% \mehdi{I wouldn't talk about that here, maube later mention that it is challenging so we always upper bound}
% \thang{dire qu en gros avec cette def bah peut y avoir backward ou forward et que ca capture pas le forgetting en tant que tel donc on prend la version au carré}
% \mehdi{also important to explain the limit of the upper bound -> doenst' capture learning after forgetting or backward positive transfer}
% Because it is not obvious to analyze the sign, we will investigate throughtout this paper about the square of the drift defined as the Catastrophic Forgetting.

\begin{defn} [Catastrophic Forgetting] \

Let $\tau_S$ (respectively $\tau_T$) be the source task (respectively target task), $\cD_{\tau_S}$ the source test set, the CF of task $\tau_S$ after 
training on all subsequent tasks up to task $\tau_T$ is defined as:

\begin{align}
    \Delta^{\tau_S \rightarrow \tau_T}(X^{\tau_S}) & = \|\delta^{\tau_S \rightarrow \tau_T}(X^{\tau_S})\|_2^2 \nonumber
    \\
    & = \sum_{(x,y)\in \cD_{\tau_S}} (f^{\star}_{\tau_T}(x)-f^{\star}_{\tau_S}(x))^{2}
\end{align}
\end{defn}

% \thang{the above expression is very general, but under NTK we can have a linear expression}
% Interestingly, under the NTK framework from \citep{bennani2020generalisation}, the CF can be expressed linearly as $\forall j=1,..,n_{\tau}$, $\tau \in [T]$:
% \begin{align*}
%      \delta^{\tau \rightarrow \tau+1}(x^{\tau}_{j})&=\langle \nabla_{\omega_{0}}f(x^{\tau}_{j}),\omega_{\tau+1}^{\star}-\omega_{\tau}^{\star} \rangle&=\phi(x^{\tau}_{j})(\omega_{\tau+1}^{*}-\omega_{\tau}^{*})
% \end{align*}
% which is nothing but the drift incurred by a sample  $x^{\tau}_{j}, \tau \in [T]$ after training on subsequent tasks up to $\tau+1$.

The above expression is very general but has an interesting linear form under the NTK regime and allows us to get 
insight on the behavior on the variation of the forgetting.
% \begin{equation*}
% \langle \nabla_{\omega_{0}}f(x^{\tau}_{j}),\omega_{\tau+1}^{\star(\vx)}-\omega_{\tau}^{\star(\vx)} \rangle=\phi(x^{\tau}_{j})(\omega_{\tau+1}^{*}-\omega_{\tau}^{*}) \in \mathbb{R}    
% \end{equation*}
% The former quantity is nothing but the amount of variation of $x^{\tau}_{j}$ after trained on task $\tau$. We can generalize this quantity for an arbitrary task $k \geq \tau+1$ as:
% \begin{align}
%     \delta^{\tau \rightarrow k}(x^{\tau}_{j})=\phi(x^{\tau}_{j})(\omega_{k}^{*}-\omega_{\tau}^{*}) , \quad j=1,..,n_{\tau}
% \end{align}
% \thang{Do I need to justify where comes the above expression from?}

% In the same logic, we  define the average drift incurred by the whole dataset $X^{\tau}$ as, $\forall k \geq \tau+1$:
% \begin{equation}
% \delta^{\tau \rightarrow k}(X^{\tau})=\phi(X^{\tau})(\omega_{k}^{*}-\omega_{\tau}^{*}) 
% \end{equation}

\begin{lemma}[CF under NTK regime] \label{lemma:def_forgetting}  \

Let $ \{ \w_{\tau}^{\star}, \forall \tau \in [T] \}$ be the weight at the end of the training of task $\tau$, the Catastrophic Forgetting of a source task $\tau_S$ with respect to a target task $\tau_T$ is given
by:
\begin{align}
\Delta^{\tau_S \rightarrow \tau_T}(X^{\tau_S})&=\norm{\delta^{\tau_S \rightarrow \tau_T}(X^{\tau_S})}_{2}^{2}  \\ &=\norm{\phi(X^{\tau_S})(\omega_{\tau_T}^{*}-\omega_{\tau_S}^{*})}_{2}^{2} \label{eq:CF_1}
% &= \norm{\displaystyle{\sum_{j=\tau_S+1}^{\tau_T}}\phi(X^{\tau_S})(\omega_{j}^{*}-\omega_{j-1}^{*})}_{2}^{2}  \label{eq:CF_2}
\end{align}
\end{lemma}

% \pierre{BIG PROBLEM: la notation $\omega_\tau^\star$ n'a pas été introduite!!!!}

\begin{proof}
See Appendix Section~\ref{proof:def_forgetting}.
\end{proof}
Lemma~\ref{lemma:def_forgetting} expresses the forgetting as a linear relation between the kernel $\phi
(X^{\tau_S})$ (which is assumed to be constant) and the variation of the weights from the source task $\tau_S$ until the target task $\tau_T$.

\begin{remark} \
Note that, from Equation \ref{eq:CF_1}, two cases are possible when $\Delta^{\tau_S \rightarrow \tau_{T}}(X^{\tau_S})=0$. The trivial case happens when $ \forall \tau \in [T]$: 
\begin{align*}
    \Bigl(f^{\star}_{\tau+1}(x)-f^{\star}_{\tau}(x)\Bigr)_{(x,y)\in \cD_{\tau_S}} = 0
\end{align*}
In this case, there is no drift at all. However, it is also possible that some tasks induce a drift on $X^{\tau_S}$ that is compensated by subsequent tasks. Indeed, for $\forall \tau \in [T]$:
\begin{align*}
    0 & = \delta^{\tau_S \rightarrow \tau_T}(X^{\tau_S})
    \\ & = \Bigl(f^{\star}_{\tau_T}(x)-f^{\star}_{\tau_S}(x)\Bigr)_{(x,y)\in \cD_{\tau_S}}
    \\ & = \Bigl(f^{\star}_{\tau_T}(x)-f^{\star}_{\tau}(x)+f^{\star}_{\tau}(x)-f^{\star}_{\tau_S}(x)\Bigr)_{(x,y)\in \cD_{\tau_S}}
\end{align*}
simply implies, for any $(x,y)\in \cD_{\tau_S}$,
$$ f^{\star}_{\tau_T}(x)-f^{\star}_{\tau}(x) =  - (f^{\star}_{\tau}(x)-f^{\star}_{\tau_S}(x)). $$
This would be an example of no forgetting due to a forward/backward transfer in the sense of~\cite{gem}.
\end{remark}
Now that we have defined the central quantity of this study, we will gain deeper insights by investigating SGD which is the vanilla algorithm.

\subsection{High correlations across tasks induce forgetting for vanilla SGD}

In this section, we derive the Catastrophic Forgetting expression for SGD. This will be the starting point to derive CF for the projection based methods (OGD, GEM and \algoname).

\begin{thm} (Catastrophic Forgetting for SGD)  \label{thm:forgetting_sgd} \
Let $U_{\tau}\Sigma_{\tau}V_{\tau}^T$ be the SVD of $\phi(X^\tau)$ for each $\tau \in [T]$,
%Let $\{U_{\tau},\Sigma_{\tau},V_{\tau}\}$, $\forall  be the SVD decomposition of $\phi(X^{\tau})$,
and let $\lambda >0$ the weight decay regularizer. The CF from task $\tau_S$ up until task $\tau_T$ is then given by:
\begin{align}
    \Delta^{\tau_S \rightarrow \tau_T}(X^{\tau_S}) =
    \norm{\sum_{k=\tau_S+1}^{\tau_T} U_{\tau_S} \Sigma_{\tau_S}O^{\tau_S \rightarrow k}_{SGD}M_{k}\tl{y}_{k}}_{2}^{2}
\end{align}
where:
\begin{align*}
    O^{\tau_S \rightarrow k}_{SGD} &=V_{\tau_S}^{\top}V_{k} \\
    M_{k}&=\Sigma_{k}[\Sigma_{k}^{2}+\lambda I_{n_{k}}]^{-1}U_{k}^{\top}  \\
    \tl{y}_{k}&=y_{k}-f_{k-1}^{\star}(x^{k})
\end{align*}
\end{thm}

\begin{proof}
See Appendix Section~\ref{proof:forgetting_sgd}.
\end{proof}

% \paragraph[]{Mehdi - Suggestion}
Theorem~\ref{thm:forgetting_sgd} describes the Catastrophic Forgetting for SGD on the task $\cT_{\tau_S}$ after
training on the subsequent tasks up to the task $\cT_{\tau_T}$.
The CF is expressed as a function of the overlap between the subspaces of the subsequent tasks and the reference
task, through what we call the \textbf{NTK overlap matrices} $\{O^{\tau_S \rightarrow k}_{SGD}, k \in [\tau_S+1, \tau_T] \}$.
High overlap between tasks increases the norm of the NTK overlap matrix which implies high forgetting.

More formally, the main elements of Catastrophic Forgetting are :
\begin{itemize}
    \item $\Sigma_{\tau_S}$ encodes the importance of the principal components of the source task. Components with high magnitude contribute to forgetting since they imply high variation along thoses directions. 
    % The more the component is high, the more a variation along contributes to forgetting.
    \item $\{O^{\tau_S \rightarrow k}_{SGD}, k \in [\tau_S+1, \tau_T] \}$ encodes the similarity of the principal
    components  between the source task and a subsequent task $k$. High norm of this matrix means high overlap between tasks and leads to high risk of forgetting. This forgetting occurs because the previous knowledge along a given component may be erased by the new dataset.
    \item $\tl{y}_{k}$ encodes the residual that remains to be learned by the current model.
    A null residual implies that the previous model predicts perfectly the new task, therefore there is no learning
    hence no forgetting.
    \item $M_{k}$ is a rotation of the residuals weighted by the principal components space.
    The rotated residuals $M_{k} \tl{y}_{k}$ can be interpreted as the residuals along each principal component.
    \item $\norm{\sum_{k=\tau_S+1}^{\tau_T} \cdot}$ encodes that the forgetting can be canceled by other tasks by learning again forgotten
    knowledge.
\end{itemize}

We will see in what follows that the matrix $O^{\tau_S \rightarrow \tau_T}_{SGD}$ captures the alignment between the source task $\tau_S$ and the target task $\tau_T$. More formally, the singular values of $O^{\tau_S \rightarrow \tau_T}_{SGD}$ are the cosines of the \emph{principal angles} between the spaces spanned by the source data $\phi(X^{\tau_S})$ and the target data $\phi(X^{\tau_T})$~\citep{wedin1983angles}.

\begin{corr} [Bounding CF with angle between source and target subspace] \label{corr:angle_forgetting_sgd} \

Let  $\Theta^{\tau_S \rightarrow \tau_T}$ be the  diagonal matrix of singular values of $O^{\tau_S \rightarrow \tau_T}_{SGD}$~(each diagonal element $\cos(\theta_{\tau_S,\tau_T})_{i}$ is  the cosine of the $i$-th principal angle between  $\phi(X^{\tau_S})$ and $\phi(X^{\tau_T})$). Let $\sigma_{\tau_S,1} \geq \sigma_{\tau_S,2} \geq...\geq \sigma_{\tau_S,n_{\tau_S}}$ be the singular values of $\phi(X^{\tau_S})$~(i.e. the diagonal elements of $\Sigma_{\tau_S}$).

The bound of the forgetting from a source task $\tau_S$ up until a target task $\tau_T$ is given by:
\begin{align}
\Delta^{\tau_S \rightarrow \tau_T}(X^{\tau_S}) \leq 
\sigma_{\tau_S,1}^{2} \sum_{k=\tau_S+1}^{\tau_T}\norm{\Theta^{\tau_S \rightarrow k}}_{2}^{2}\norm{M_{k}\tilde{y}_{k}}_{2}^{2}  
\end{align}
\end{corr}
\begin{proof}
See Appendix Section~\ref{proof:angle_forgetting_sgd}. 
\end{proof}
Corollary~\ref{corr:angle_forgetting_sgd} bounds the CF by the sum of the cosines of the first principal angles between the source task $\tau_S$ and each subsequent task until the target task $\tau_T$~(represented by the diagonal matrix $\Theta^{\tau_S \rightarrow k}$) and a coefficient $\sigma_{\tau_S,1}^{2}$ from the source task $\tau_S$. 
\begin{itemize}
    \item $\{ \Theta^{\tau_S \rightarrow k} , k \in [\tau_S+1, \tau_T] \}$ is the diagonal matrix where each element represents the cosine angle between subspaces $\tau_S$ and $k$: $\cos(\theta_{\tau_S,k})_{i}$. If the principal angle between two tasks is small~(i.e. the two tasks are aligned), the cosine will be large which implies a high risk of forgetting.
    \item $\sigma_{\tau_S,1}$ is the variance of the data of task $\tau_S$ along its principal direction of variation. Intuitively,  $\sigma_{\tau_S,1}$ measures the spread of the data for task $\tau_S$.
\end{itemize}

% \mehdi{not super convinced about the point of stating the bound, we can get a similar interpretation from the previous result right ?}

In the end, a potential component responsible for CF in the Vanilla SGD case is the projection from the source task onto the target task. This phenomenon is best characterized by the eigenvalues of $O^{\tau_S \rightarrow \tau_T}_{SGD}$ which acts as a similarity measure between the tasks. One avenue to mitigate the CF can be to project orthogonally to the source task subspace which are the main insight from OGD
\citep{farajtabar2020orthogonal} and GEM \citep{gem}.
% \mehdi{Not sure it is worth presenting experiments now !}
% We will show through abalation analysis that the more similar tasks are, the higher those eigenvalues will be and higher forgetting will be incurred.

\subsection{The effectiveness of the orthogonal projection against Catastrophic Forgetting}

Now, we study the GEM and OGD algorithms, we identify these two algorithms as projection based algorithms.
We extend the previous analysis to study the effectiveness of these algorithms against Catastrophic Forgetting.

\paragraph[]{Recap}

OGD \citep{farajtabar2020orthogonal} stores the feature maps of arbitrary samples from each task, then projects the
update gradient orthogonally to these feature maps. The idea is to preserve the subspace spanned by the previous samples (\citep{yu2020gradient} proposed a similar variant for multi-task learning ).

GEM \citep{gem} computes the gradient of the train loss over each previous task, by storing samples from each task.
While OGD performs an orthogonal projection to the \textbf{gradients} of the model, GEM projects orthogonally to the space spanned by the \textbf{losses gradients}.
The idea is to update the model under the constraint that the train loss over the previous tasks does not increase.

\paragraph[]{GEM-NT : Decoupling Forward/Backward Transfer from Catastrophic Forgetting }

OGD has been extensively studied by \cite{bennani2020generalisation}), therefore we perform the analysis for the GEM
algorithm, then highlight the similarities with OGD.
Also, in order to decouple CF from Forward/Backward Transfer, we study a variant of GEM with no transfer at all, which we call GEM No Transfer (GEM-NT).  

Similarly to GEM, GEM-NT maintains an episodic memory containing $d$ samples from each previous tasks seen so far.
During each gradient step of task $\tau+1$, GEM samples from the memory $d$ elements from each previous task then compute the average loss function gradient:
\begin{align*}
    g_{k}=\frac{1}{d}\displaystyle{\sum_{j=1}^{d}}\nabla_{\w}\cL^{k}_{\lambda}(x^{k}_{j}), \quad  \forall k=1,..,\tau
\end{align*}

If the proposed update during task $\tau+1$ can potentially degrades former solutions (i.e $\langle g_{\tau+1} , g_{k} \rangle <0 , \forall k \leq \tau$) then the proposed update is projected orthogonally to these gradients $g_{k}$, $\forall k \leq \tau$.

As opposed to GEM, which performs the orthogonal projection conditionally on the impact of the gradient update on the
previous training losses, GEM-NT project orthogonally to $g_{k}$, $\forall k \leq \tau$ at each step \textbf{irrespectively} of the sign of the dot product. The algorithm pseudo-code can be found in Appendix Section~\ref{alg:gem-nt}.

\paragraph[]{The effectiveness of GEM-NT against CF}
Denote $G_{\tau} \in \mathbb{R}^{p \times \tau}$ the matrix where each columns represents $g_{k}, \forall k=1,.
.,\tau$, the orthogonal projection matrix is then defined as
$T_{\tau}=I_{p}-G_{\tau}G_{\tau}^{\top}=\overline{G}_{\tau}\overline{G}_{\tau}^{\top}$. This represents an orthogonal projection whatever the sign of the dot product $\langle g_{\tau+1} , g_{k} \rangle$ in order to decouple the forgetting from transfer.

We are now ready to provide the  CF of GEM-NT.

\begin{corr} [CF for GEM-NT] \label{corr:cf_gem_nt} \

Using the previous notations. The CF from task $\tau_S$ up until task $\tau_T$ for GEM-NT given by:
\begin{align}
     \Delta^{\tau_S \rightarrow \tau_T}(X^{\tau_S}) = \norm{\sum_{k=\tau_S+1}^{\tau_T} U_{\tau_S}\Sigma_{\tau_S}\textcolor{greencustom}{\mathbf{O^{\tau_S \rightarrow k}_{\text{GEM-NT}}} M_{k}}\tl{y}_{k}}_{2}^{2}
\end{align}
where:
\begin{align*}
    O^{\tau_S \rightarrow k}_{\text{GEM-NT}}&= V_{\tau_S}^{\top}\textcolor{greencustom}{\overline{G}_{k-1}\overline{G}_{k-1}^{\top}}V_{k} \\
     M_{k}&=\Sigma_{k}U_{k}^{\top}[\textcolor{greencustom}{\overline{\phi}(X^{k})\overline{\phi}(X^{k})^{\top}}+\lambda I_{n_{k}}]^{-1}  \\
     \overline{\phi}(X^{k})&= \phi(X^{k})T_{k-1} 
\end{align*}
(Differences with the vanilla case  SGD are highlighted in color)
\end{corr}

\begin{proof}
See Appendix Section~\ref{proof:gem_nt}.
\end{proof}

The difference for GEM-NT lies in the \textbf{double} projection of the source and target task onto the subspace $\overline{G}_{\tau}$ which contain elements orthogonal to $g_{k}, \forall k=1,..,\tau-1$.

Similarly to Corollary~\ref{corr:angle_forgetting_sgd}, we can bound each projection matrix ($V_{\tau_S}^{\top}\overline{G}_{k-1}$ and $\overline{G}_{k-1}^{\top}V_{k}$,$\forall k \in [\tau_S+1,\tau_T]$ ) by their respective matrices of singular values ($\Theta^{\tau_S \rightarrow G_{k-1}}$ and $\Theta^{k \rightarrow G_{k-1}}$, $\forall k \in [\tau_S+1,\tau_T]$). This leads us to the following upper-bound for the CF of GEM-NT:
\begin{align}\label{eq:bound-GEM-TN}
    &\Delta^{\tau_S \rightarrow \tau_T}(X^{\tau_S})\leq \\
     &\sigma_{\tau_S,1}^{2}\sum_{k=\tau_S+1}^{\tau_T}  \textcolor{greencustom}{\norm{\Theta^{\tau_S \rightarrow \overline{G}_{k-1}}}_{2}^{2}\norm{\Theta^{k \rightarrow \overline{G}_{k-1}}}_{2}^{2}}\norm{M_{k}\tl{y}_{k}}_{2}^{2}\nonumber
\end{align}

\paragraph{Connection of GEM-NT to OGD} \label{remark:comparison_ogd_gem_n}
For the analysis purpose, let's suppose that the memory per task is $1$, $\lambda=0$, $\forall \tau \in [T]$ and assume a mean square loss error function. In that case:
\begin{align}
    g_{k}=  
    \begin{cases}
    \nabla_{\w}f_{k}(x)(f_{k}(x)-y_{k}) \quad &(\text{GEM-NT})\\
    \nabla_{\w}f_{k}(x)  \quad \quad &(\text{OGD})
    \end{cases}
\end{align}
% $g_{\tau}^{GEM-NT}=\nabla_{\w}f(x)\underline{(f_{\tau}(x)-y_{\tau})}$ while OGD has $g_{\tau}^{OGD}=\nabla_{\w}f(x)$.
\begin{itemize}
    \item unlike OGD, GEM-NT weights the orthogonal projection with the residuals $(f_{\tau}(x^{k})-y_{k})=(\tl{y}_{k}+\delta^{k \rightarrow \tau}(x^{k}))$ which represents the difference between the new prediction (due to the drift) for $x^{k}$ under model $\tau$ and the target $y_k$.
    \item Previous tasks that are well learned (small residuals) will contribute less to the orthogonal projection to the detriment of tasks with large residuals (badly learned then). This seems counter-intuitive because by doing so, the projection will not be orthogonal to well learned tasks (in the edge case of zero residuals) then unlearning can happen for those tasks. 
\end{itemize}

% OGD and GEM-N both sample random element from each tasks $\tau \in [T]$ and use them for projection purpose. If data have a structure one can go one step further by compression that information into a reduce number of component through Principal Component Analysis (PCA) and in same time meet the memory constraint criterion.

While OGD and GEM-NT are more robust to CF than SGD through the orthogonal projection, they do not leverage explicitely the structure in the data \citep{farquhar2018towards}. We can then compress this information through dimension reduction algorithms such as SVD in order to both maximise the information contained in the memory as well as mitigating the CF.

% to maximise the potential of the information in the memory. Considering that the data is structured, 

% aleviate more CF.

\subsection{\algoname: leveraging structure by projecting orthogonally to the top d principal directions}

% Although \citep{bennani2020generalisation} proves no forgetting for OGD under infinite memory case, this hides the more realistic finite memory case. In the next section, we will provide some insight in the latter case through the lens of PCA.

% Let's denote the total memory size $M$, each task is then allocated a memory of $M/T=d$.
Unlike OGD that stores randomly $d$ samples from each task $k=1,..,\tau$ of $\{ \nabla_{\w}f(x^{k}_{j}) \}_{j=1}^{d}$, at the end of each task $\tau$, \algoname samples randomly $s>d$ elements from $X^{\tau}$ then stores the top $d$ eigenvectors of $\{ \nabla_{\w}f(X^{\tau})\}$ denoted as $v_{\tau,i}$, $i=1,..,d$. These are the directions that capture the most variance of the data. If we denote by $P_{\tau,:d}$ the matrix where each columns represents $v_{k,i}$, $k=1,..,\tau$, $i=1,..,d$ then the orthogonal matrix projection can be written as:
\begin{align}
    T_{\tau,:d}=I_{p}-P_{\tau,:d}P_{\tau,:d}^{\top}=R_{\tau,d:}R_{\tau,d:}^{\top}
\end{align}
where the columns of $R_{\tau,d:}$ form an orthonormal basis of the orthogonal complement of the span of $P_{\tau,:d}$.
For the terminology, $P_{\tau,:d} \in \mathbb{R}^{p \times (\tau \cdot d)}$ (respectively $R_{\tau,d:} \in \mathbb{R}^{p \times p-(\tau \cdot d)}$) represents the \textbf{top subspace}  (respectively the \textbf{residuals subspace}) of order $d$ for task $1$ until $\tau$. A pseudo-code of \algoname is given in Alg. \ref{alg:ogd-pca} (the computational overhead can be found in the Appendix). We are now ready to provide the CF of \algoname.

\begin{algorithm}[h!]
    \SetKwInOut{Input}{Input}
    \SetKwInOut{Output}{Output}
    \SetKw{KwBy}{by}
    \SetAlgoLined
    \Input{A task sequence $\cT_1, \cT_2, \ldots  $, learning rate $\eta$, PCA samples $s$, components to keep $d$}
%    \Output{The optimal parameter $\vw^\star$}
    \begin{enumerate}
        \item Initialize $S_{J}\leftarrow \{\}$ ; $\w \leftarrow \w_0$
        \item \For{Task ID $\tau=1,2,3, \ldots $}{
        \Repeat{convergence}{
        
        $\vg \leftarrow$ Stochastic Batch Gradient for $\cT_\tau$ at $\w$\;
        \tcp{Orthogonal updates}
        $\tilde{\vg} = \vg - \sum_{\vv \in \cS_J} \mathrm{proj}_{\vv}(\vg)$\;
        $\w \leftarrow \w - \eta \tilde{\vg}$
        }
        \tcp{Gram-Schmidt orthogonalization}
        \For{$(\vx, y) \in \cD_\tau $ \text{and} $k \in [1, c]$ \text{s.t.} $y_k = 1$}{
        $\vecu \leftarrow \nabla f_\tau(\vx; \w) - \sum_{\vv \in \cS_J} \mathrm{proj}_{\vv}(\nabla_{\w} f_\tau(\vx;
        \w))$
        $\cS_J \leftarrow \cS_J \bigcup \{\vecu \} $
        }
    \tcp{PCA}
        {\color{red} Sample $s$ elements from $\cT_{\tau}$}  \\
        {\color{red} top $d$ eigenvectors $ \leftarrow PCA (\{ \nabla_{\w} f_{\tau}(x^{\tau}_{j}) \}_{j=1}^{s} )$} \  
        {\color{red} $\cS_J \leftarrow  \cS_J \bigcup  \ \{ $ top $d$ eigenvectors $\}$} \ 
        }
        % {\color{red} ${\color{red} \cS_J \leftarrow PCA(\cS_{d}) } $}}
    \end{enumerate}
    \caption{ 	 \algoname (Differences with OGD in \color{red}red)}
    \label{alg:ogd-pca}
\end{algorithm}

\begin{corr}[Forgetting for \algoname] \label{corr:forgetting_pca} \

% Let's denote $\tl{\phi}(X^{\tau})=\phi(X^{\tau})T_{\tau-1,:d}$, $\{ U_{\tau,d:},\Sigma_{\tau,d:},\Sigma_{\tau,d:} \} , \tau \in [T]$ the matrices corresponding to the truncated SVD at order $d$ of $\phi(X^{\tau})$. The CF for \algoname is given by: 
For each $\tau \in [T]$, let $\tl{\phi}(X^{\tau})=\phi(X^{\tau})T_{\tau-1,:d}$ and let $U_{\tau}\Sigma_{\tau}V_{\tau}^T$ be the  SVD of $\phi(X^{\tau})$. The CF for \algoname is given by:
% \guillaume{Rephrasing suggestion:
% For each $\tau \in [T]$, let $\tl{\phi}(X^{\tau})=\phi(X^{\tau})T_{\tau-1,:d}$ and let $U_{\tau,d:}\Sigma_{\tau,d:}V_{\tau,d:}^T$ be the rank $d$ truncated SVD of $\phi(X^{\tau})$. The CF for \algoname is given by: }
\\
\begin{align} \label{thm:forgetting_pca_ogd} 
&  \Delta^{\tau_S \rightarrow \tau_T}(X^{\tau_S}) = \norm{\sum_{k=\tau_S+1}^{\tau_T}  \textcolor{bluecustom}{\mathbf{U_{\tau_S}\Sigma_{\tau_S} O^{\tau_S \rightarrow k}_{PCA}} M_{k}}\tl{y}_{k} }_{2}^{2}  
\end{align}
where:
% \begin{equation*}
\begin{align*}
% M^{j}&=U^{j}(\Sigma^{j})^{2}(U^{j})^{\top}+F^{j-1} \\
O^{\tau \rightarrow k}_{PCA} &=V_{\tau_S}^{\top}\textcolor{bluecustom}{R_{k-1,d:}R_{k-1,d:}^{\top}}V_{k} \\
 M_{k}&=\Sigma_{k}U_{k}^{\top}[\textcolor{bluecustom}{\tl{\phi}(X^{k})\tl{\phi}(X^{k})^{\top}}+\lambda I_{n_{k}}]^{-1}  \\
\tl{\phi}(X^{k})&=\phi(X^{k})T_{k-1,:d}
\end{align*}
\end{corr}

\begin{proof}
See Appendix Section~\ref{proof:forgetting_pca}.
\end{proof}

Corollary~\ref{corr:forgetting_pca} underlines the difference with GEM-NT as this time the double projection are on the residuals subspace $R_{k-1,:d}$ containing the orthogoanl vector to the features map $\nabla_{\w}f(x)$ instead of the loss function gradient. 

\begin{remark} \ 

\begin{itemize}
\item PCA is helpful in datasets where the eigenvalues are decreasing exponentially since keeping a small number of components can leverage a large information and explain a great part of the variance. Projecting orthogonally to these main components will lead to small forgetting if $\sigma_{\tau,d+1}$ is small.
\item On the other hand, unfavourable situations where data are spread uniformly along all directions (i.e, eigenvalues are uniformly equals ) will requires to keep all components and a larger memory. As an example, we build a worst-case scenario in Appendix Section \ref{toy:worst_case_pca} where OGD is performing better than PCA-OGD.
\end{itemize}
\end{remark}

% \begin{corr}[Angle with the residual subspace] \label{corr:upper_bound_forgetting_pca} \

% Let's denote the diagonal matrix from the SVD decomposition of $V_{\tau_S,d:}^{\top}R_{k,:d}$, $k \geq \tau_S+1$ as $\Theta^{\tau_S \rightarrow R_{k,:d}}$ where each diagonal element is $\cos(\theta^{\tau_S,R_{k,:d}}_{i})$. Without loss of generality, assume the eigenvalues are ordered in decreasing order:
% $\cos(\theta^{\tau_S,R_{k,:d}}_{1}) \geq \cos(\theta^{\tau_S,R_{k,:d}}_{2}),...$.

% The bound of CF for \algoname is given by:
Similarly as the previous case, we can bound the double projection on $R_{k-1,:d}$ with the corresponding diagonal matrix $\Theta^{\tau_S \rightarrow R_{k,:d}}$. Additionally, the CF is bounded by $\sigma_{\tau_S,d+1}$ which is due to the orthogonal projection to the first $d$ principal directions. The upper bound of the CF is given by:
\begin{align}
&\Delta^{\tau_S \rightarrow \tau_T}(X^{\tau}) \leq  \\
&\textcolor{bluecustom}{\mathbf{\sigma_{\tau_S,d+1}^{2}}} \sum_{k=\tau_S+1}^{\tau_T}\textcolor{bluecustom}{\mathbf{\norm{\Theta^{\tau_S \rightarrow R_{k-1,:d}}}_{2}^{2}\norm{\Theta^{k \rightarrow R_{k-1,:d}}}_{2}^{2}}}\norm{M_{k}\tl{y}^{k}}_{2}^{2}  \nonumber
\end{align}

Note that in contrast with Eq.~\eqref{eq:bound-GEM-TN}, the first term in the upper bound is the ($d+1$)-th singular value of $\phi(X^{\tau_S})$, which is due to the PCA step of \algoname. 
A summary of the forgetting properties of the described methods can be found in Table~\ref{tab:summary_alg} in Appendix.

\section{Experiments}

In this section, we study the impact of the NTK overlap matrix on the forgetting by validating Corollary~\ref{corr:angle_forgetting_sgd}. We then illustrates how \algoname efficiently captures and compresses the information in datasets (Corollary~\ref{corr:forgetting_pca}. Finally, we benchmark \algoname on standard CL baselines.

\subsection{Low eigenvalue of the NTK overlap matrix induces smaller drop in performance}

\subparagraph{Objective :} As presented in Corollary~\ref{corr:angle_forgetting_sgd}, we want to assess the effect of the eigenvalues of the NTK overlap matrix on the forgetting. 

\subparagraph{Experiments :} We measure the drop in accuracy for task $1$ until task $15$ on Rotated MNIST with respect to the maximum eigenvalue of the NTK overlap matrix $O^{1 \rightarrow 15}$. 

\subparagraph{Results :} Figure~\ref{fig:drop_in_performance} shows the drop in accuracy between task $1$ and task $15$ for Rotated MNIST versus the largest eigenvalue of $O^{1 \rightarrow 15}$. As expected low eigenvalues leads to a smaller drop in accuracy and thus less forgetting. \algoname improves upon OGD, having from $7\%$ to $10\%$ less drop in performance.

% \mehdi{Suggestion :}
% In Thm ....; we state that the NTK overlap matrix is the key to CF .......
% \subparagraph{Experiment : }
% We assess the effect of the eigenvalues of the NTK overlap matrix on the forgetting.
% \subparagraph{Results : }
% Figure~\ref{fig:drop_in_performance} shows the drop in accuracy between task $1$ and task $15$ for Rotated MNIST versus the eigenvalues of $O^{1 \rightarrow 15}$ for different memory size: $25,50,100,200$. As expected lower eigenvalues leads to a smaller drop in accuracy and then less forgetting. \algoname has roughly $10\%$ less drop un performance that OGD for a memory size of $200$ and $7\%$ for a memory size of $25$.

\begin{figure}[ht]
   \hspace{-0.5em}
    \includegraphics[width=0.9\linewidth]{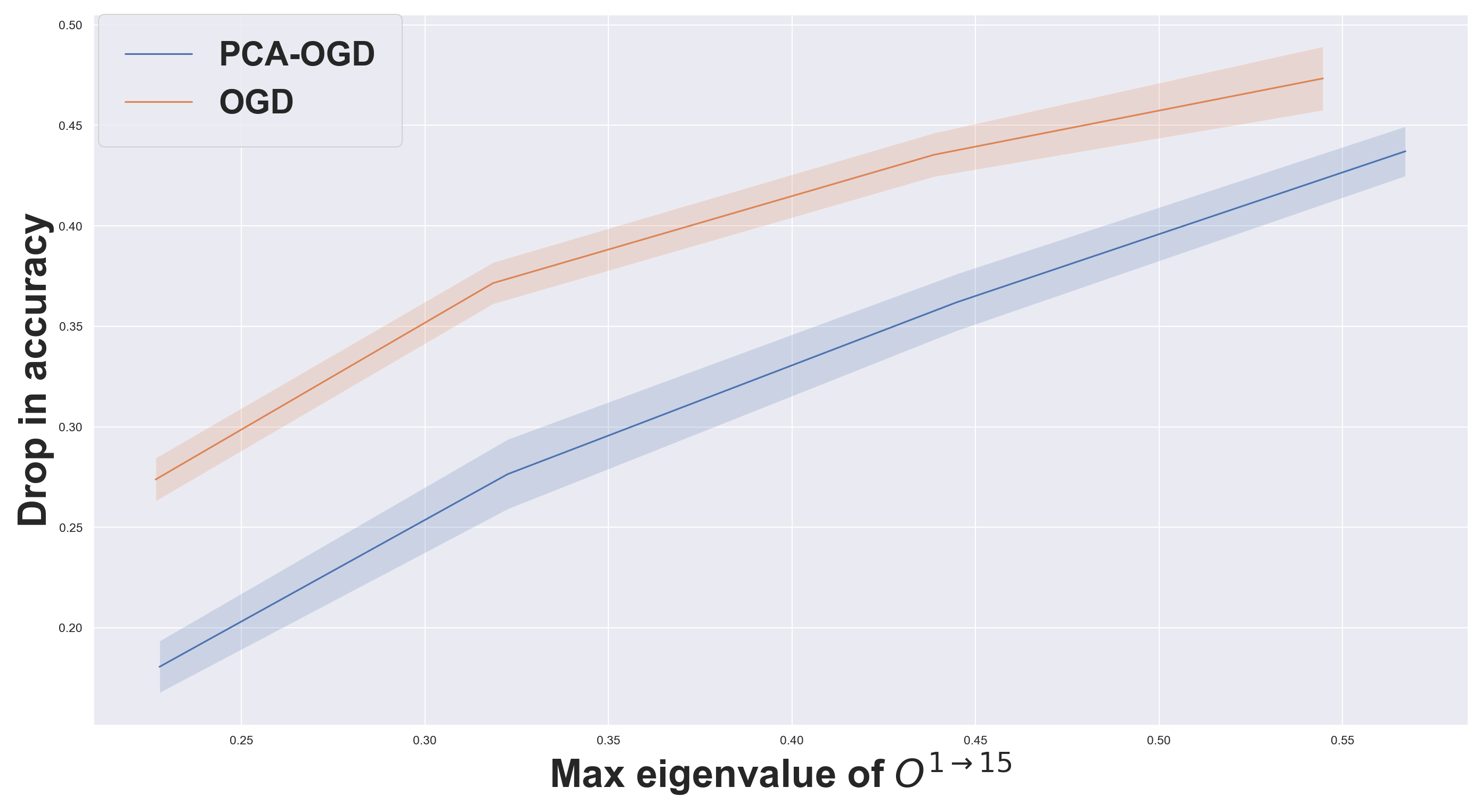}
    \caption{Drop in performance with respect to the maximum eigenvalue for Rotated MNIST (averaged over $5$ seeds $\pm 1$ std).}
    \label{fig:drop_in_performance}
\end{figure}

\subsection{\algoname reduces forgetting by efficiently leveraging structure in the data}

\subparagraph{Objective :} We show how capturing the top $d$ principal directions helps reducing Catastrophic Forgetting (Corollary~\ref{corr:forgetting_pca}).

\subparagraph{Experiments :} We compare the spectrum of the NTK overlap matrix for different methods: SGD, GEM-NT, OGD and \algoname, for different memory sizes.

\subparagraph{Results: } We visualize the effect of the memory size on the forgetting through the eigenvalues of the NTK overlap matrix $O^{\tau_S \rightarrow \tau_T}$. To unclutter the plot, Figure~\ref{fig:rotated_comparison_wrt_buffer_size} only shows the results for memory sizes of $25$ and $200$. Because \algoname compresses the information in a few number of components, it has lower eigenvalues than both OGD and GEM-NT and the gap gets higher when increasing the memory size to $200$. Table~\ref{fig:dataset_explained_variance} in the Appendix confirms those findings by seeing that with $200$ components one can already explain $90.71 \%$ of the variance. 

Finally, the eigenvalues of SGD are higher than those of projection methods since it does not perform any projection of the source or target task.

% A counter-example dataset (Permuted MNIST) where there is no patterns across datasets is shown in Appendix Section.

\begin{figure}[ht]
    %  \centering
    \hspace{-1em}
    \includegraphics[width=1.05\linewidth]{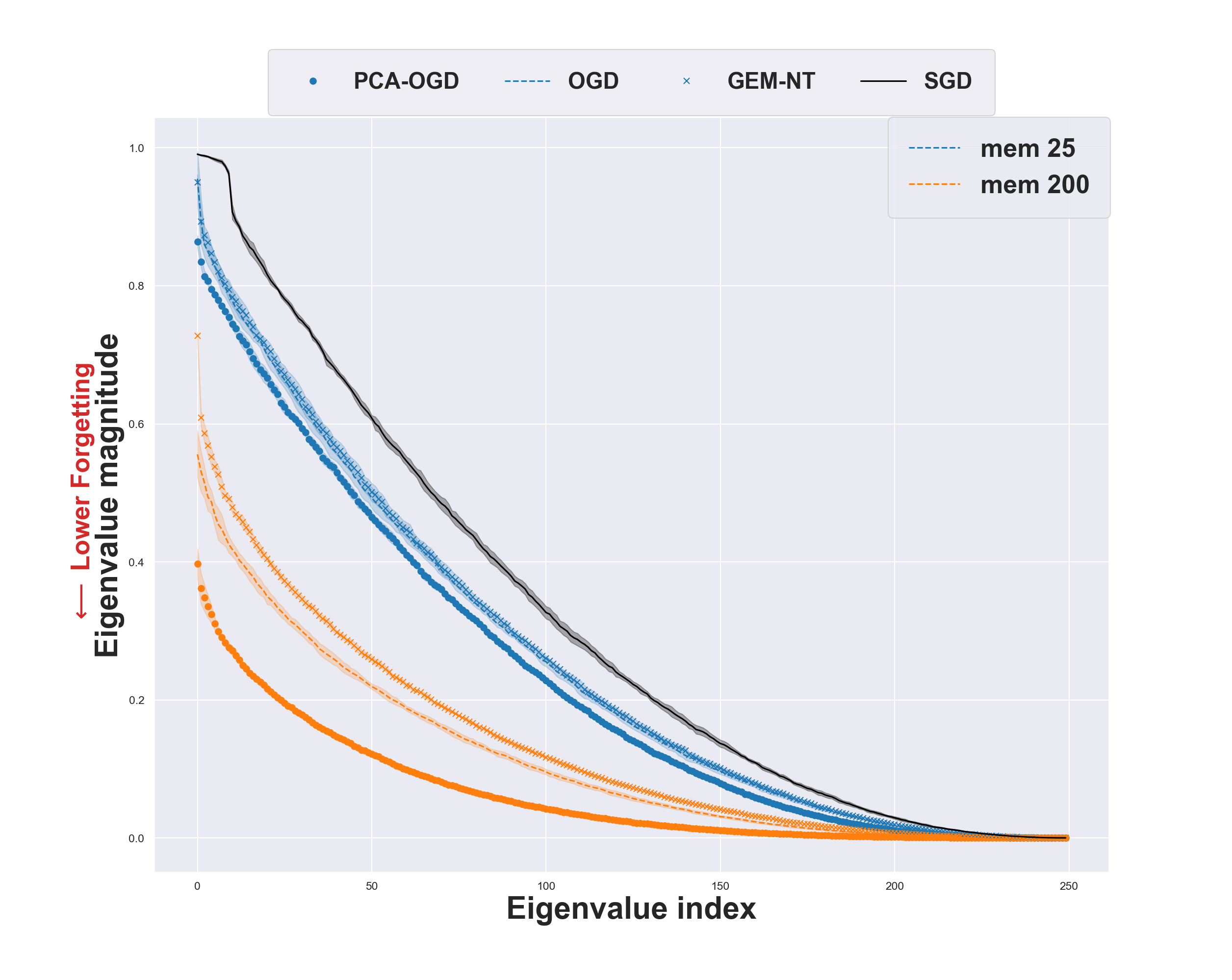}
    \caption{Comparison of the eigenvalues of $O^{1 \rightarrow 2}$ on \textbf{Rotated MNIST} with increasing memory size. Lower values imply less forgetting (averaged over $5$ seeds $\pm 1$ std).}
    \label{fig:rotated_comparison_wrt_buffer_size}
\end{figure}

Finally, the final accuracies on Rotated and Permuted MNIST are reported in~Figure~\ref{fig:rotated_bar2} for the first seven tasks. In Rotated MNIST, we can see that \algoname is twice more memory efficient than OGD: with a memory size of $100$ \algoname has comparable results to OGD with a memory size $200$. Interestingly, while the marginal increase for \algoname is roughly constant going from memory size $25$ to $50$ or $50$ to $100$, OGD incurs a high increase from memory size $100$ to $200$ while below that threshold the improvement is relatively small.

\begin{figure}[ht]
   \hspace{-0em}
    \includegraphics[width=1.0\linewidth]{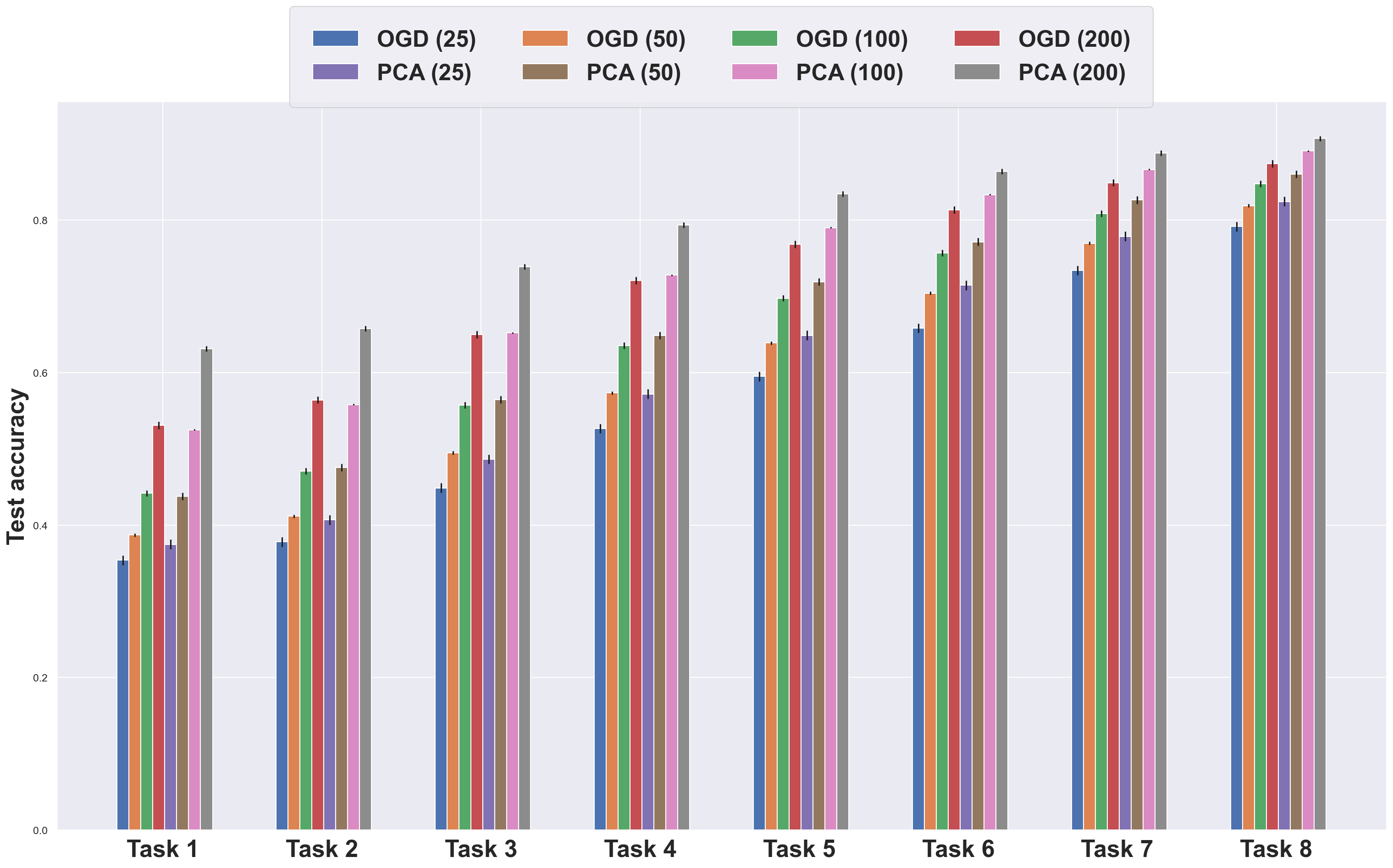}
    \caption{Final accuracy on \textbf{Rotated MNIST} for different memory size (averaged over $5$ seeds $\pm 1$ std). OGD needs twice as much memory as \algoname in order to achieve the same performance (i.e compare OGD (200) and PCA (100).}
    \label{fig:rotated_bar2}
\end{figure}

We ran OGD and PCA-OGD on a counter-example dataset (Permuted MNIST), where there is no structure within the dataset (see Appendix~\ref{sec:permuted_mnist}). In this case, \algoname is less efficient since it needs to keep more principal components than in a structured dataset setting.

\subsection{General performance of \algoname against baselines}

\subparagraph{Objective and Experiments :} We compare \algoname against other baseline methods: SGD, A-GEM \citep{agem} and OGD \citep{farajtabar2020orthogonal}. Additionally to the final accuracies, we report the \textbf{Average Accuracy} $A_{T}$ and \textbf{Forgetting Measure} $F_{T}$~\citep{gem,agem}. We run AGEM instead of GEM-NT which is faster with comparable results~\citep{agem} (since GEM-NT is solving a quadratic programming optimization at each iteration step). Definition of these metrics and full details of the experimental setup can be found in Appendix~\ref{app:experiments_details}.

\subparagraph{Results :} The results are summarized in Table~\ref{tab:general_results} (additional results are presented in Appendix~\ref{app:experiments_details}). Overall, \algoname obtains comparable results to A-GEM. A-GEM has the advantage of accounting for the NTK changes by updating it while \algoname and OGD are storing the gradients from previous iteration. The later therefore project updates orthogonally to outdated gradients. This issue has also been mentioned in \citep{bennani2020generalisation}. Note the good performance of \algoname in Split CIFAR where the dataset size is $2,500$~(making the NTK assumption more realistic) and similar patterns are seen across tasks (CIFAR100 dataset is divided into $20$ superclasses within which we can count $5$ subfamilies hence having a pattern across tasks. To examine this hypothesis, we plot the NTK changes for different datasets in Appendix~\ref{sec:ntk_variation}. We can indeed see that the NTK does not vary anymore after $1$ task for Split CIFAR while it increases linearly for MNIST datasets which confirms our hypothesis.

\begin{table}[ht]
    \centering
\resizebox{0.46\textwidth}{!}{
\begin{tabular}{c|c c  c c c }
\hline\hline
     & SGD  & EWC & A-GEM & OGD & \algoname  \\ \hline
     &  \multicolumn{4}{c}{Permuted MNIST} \\ \hline
 $A_{T}$  & 76.81 $\pm$ 1.36  & 79.71 $\pm$ 0.52  &  \textbf{ 83.4 $\pm$ 0.43}  & 80.95 $\pm$ 0.5  & 81.44 $\pm$ 0.62 \\ 
 $F_{T}$  & 14.88 $\pm$ 1.64    &  \textbf{ 3.81 $\pm$ 0.47 } & 7.29 $\pm$ 0.45    & 9.72 $\pm$ 0.51    & 9.11 $\pm$ 0.65    \\ \hline
             &  \multicolumn{4}{c}{Rotated MNIST} \\ \hline
   $A_{T}$  & 66.07 $\pm$ 0.47  & 76.2 $\pm$ 0.62  &  \textbf{ 83.52 $\pm$ 0.22}  & 77.42 $\pm$ 0.35  & 82.05 $\pm$ 0.58 \\ 
 $F_{T}$  & 29.57 $\pm$ 0.56    & 13.44 $\pm$ 0.82    &  \textbf{ 9.86 $\pm$ 0.28 } & 16.52 $\pm$ 0.46    & 11.67 $\pm$ 0.65    \\ \hline
        &  \multicolumn{4}{c}{Split MNIST} \\ \hline
  $A_{T}$  & 95.1 $\pm$ 1.08  & 95.06 $\pm$ 1.15  & 94.25 $\pm$ 1.62  &  \textbf{ 96.05 $\pm$ 0.34}  & 95.96 $\pm$ 0.29 \\ 
 $F_{T}$  & 2.02 $\pm$ 1.48    & 2.08 $\pm$ 1.56    & 2.82 $\pm$ 1.72    & 0.37 $\pm$ 0.21    &  \textbf{ 0.28 $\pm$ 0.15 } \\ \hline
              &  \multicolumn{4}{c}{Split CIFAR} \\ \hline
 $A_{T}$  & 56.11 $\pm$ 1.65  & 65.45 $\pm$ 0.97  & 47.55 $\pm$ 2.4  & 69.77 $\pm$ 0.72  &  \textbf{ 72.7 $\pm$ 0.97} \\ 
 $F_{T}$  & 22.69 $\pm$ 2.18    & 6.56 $\pm$ 1.49    & 31.36 $\pm$ 2.58    & 8.27 $\pm$ 0.31    &  \textbf{ 5.39 $\pm$ 0.85 } \\ \hline
\end{tabular}}
    \caption{Average Accuracy and Forgetting for all baselines considered across the datasets ($5$ seeds).}
    \label{tab:general_results}
\end{table}

\section{Conclusion}
% \thang{to be completed}
% In this paper, we define expression for CF, under NTK regime, we study CF for a family of projection based methods, derived PCA-OGD as an extension and shows importance of leverage data structure.

% Future work: either studying forward/backward transfer or update of the NTK 

We present a theoretical analysis of CF in the NTK regime, for SGD and the projection based algorithms OGD, GEM-NT and \algoname. We quantify the impact of the tasks similarity on CF through the NTK overlap matrix. Experiments support our findings
that the overlap matrix is crucial in reducing CF and our proposed method \algoname efficiently mitigates CF. However, our analysis relies on the core assumption of overparameterisation, an important next step is to account for the change of NTK over time. We hope this analysis opens new directions to study the properties of Catastrophic Forgetting for other
Continual Learning algorithms.

\section{Acknowledgments}
The authors would like to thank Joelle Pineau for useful discussions and feedbacks. Finally, we thank Compute Canada for providing computational ressources used in this project.

% \bibliographystyle{plain}

% \bibliography{bibliography.bib}

% \subsubsection*{Acknowledgements}
% All acknowledgments go at the end of the paper, including thanks to reviewers who gave useful comments, to colleagues who contributed to the ideas, and to funding agencies and corporate sponsors that provided financial support.

\onecolumn
\section{Appendix}
% \subsection{TO BE DETERMINED}

% \mehdi{Suggestion below}
% \paragraph{Interpretation}
% The weight change is closely related to the similarity of the tasks, which is captured by the weighted sum of the dot
% products of their eigen values.
% The weight of the sum corresponds to the eigen value which captures the importance of each dimension.

% \begin{proof}
% See Appendix Section \ref{app:ogd_pca_solution}.
% \end{proof}
\subsection{Proof of Lemma~\ref{lemma:def_forgetting}}  \label{proof:def_forgetting}
For this proof, we will use the result of Thm. 1 from \citep{bennani2020generalisation} (particularly Remark 1) and notice that the expression of $f_{\tau_T}^{\star}$ can be espressed recursively with respect to $f_{\tau_S}^{\star}$:
\begin{proof}
 \begin{align*}
     f_{\tau_T}^{\star}(x)&=f_{\tau_T-1}^{\star}(x)+\langle \nabla_{\w}f_{\tau_T-1}^{\star}(x), \w_{\tau_T}^{\star}-\w_{\tau_T-1}^{\star}\rangle \\
     &=f_{\tau_T-k}^{\star}(x) + ... + \langle \nabla_{\w}f_{\tau_T-2}^{\star}(x), \w_{\tau_T-1}^{\star}-\w_{\tau_T-2}^{\star}\rangle +\langle \nabla_{\w}f_{\tau_T-1}^{\star}(x), \w_{\tau_T}^{\star}-\w_{\tau_T-1}^{\star}\rangle \\
     &=f_{\tau_S}^{\star}(x) +\sum_{k=\tau_S+1}^{\tau_T} \langle \nabla_{\w}f_{k}^{\star}(x), \w_{k}^{\star}-\w_{k-1}^{\star}\rangle \\
         &=f_{\tau_S}^{\star}(x) +\sum_{k=\tau_S+1}^{\tau_T} \langle \nabla_{\w}f_{0}(x), \w_{k}^{\star}-\w_{k-1}^{\star}\rangle  \quad \text{(NTK constant)} \\
         &=f_{\tau_S}^{\star}(x) +\langle \nabla_{\w}f_{0}(x), \w_{\tau_T}^{\star}-\w_{\tau_S}^{\star}\rangle 
 \end{align*}
 where we used constant NTK assumption, i.e $\nabla_{\w}f_{\tau}^{\star}(x)=\nabla_{\w_{0}}f(x)$ , $\forall \tau \in [T]$.
 
Using the fact that the kernel is given by $\phi(x)=\nabla_{w_{0}}f(x)$, we have that:
\begin{align*}
    \delta^{\tau_S \rightarrow \tau_T}(X^{\tau_S})&=f_{\tau_T}^{\star}(X^{\tau_S})-f_{\tau_S}^{\star}(X^{\tau_S}) \\
    &=\langle \phi(X^{\tau_S}), \w_{\tau_T}^{\star}-\w_{\tau_S}^{\star} \rangle \\
     \Delta^{\tau_S \rightarrow \tau_T}(X^{\tau_S})&= \norm{\phi(X^{\tau_S}) (\w_{\tau_T}^{\star}-\w_{\tau_S}^{\star})}_{2}^{2}
\end{align*}
This concludes the proof.
\end{proof}
 
\subsection{Proof of Theorem~\ref{thm:forgetting_sgd}}  \label{proof:forgetting_sgd}
For this proof, we will decompose the drift from task $\tau_S$ to $\tau_T$ into a telescopic sum. We will then use SVD to factorize the expression of $(\w_{\tau}^{\star}-\w_{\tau-1}^{\star})$ and get the upper bound showed.
\begin{proof}
\begin{align}
    \Delta^{\tau_S \rightarrow \tau_T}(X^{\tau_S})&=\norm{\phi(X^{\tau_S})(\w_{\tau_T}^{\star}-\w_{\tau_S}^{\star})}_{2}^{2}  \quad (\text{Lemma~\ref{lemma:def_forgetting}}) \\
    &=\norm{\sum_{k=\tau_S+1}^{\tau_T}\phi(X^{\tau_S})(\w_{k}^{\star}-\w_{k-1}^{\star})}_{2}^{2} \\
     & = \norm{ \sum_{k=\tau_S+1}^{\tau_T} \phi(X^{\tau_S})\underbrace{\phi(X^{k})^{\top}[\phi(X^{k})\phi(X^{k})^{\top}+\lambda I_{n_{k}}]^{-1}\tl{y}_{k}}_{\text{from Thm. 1 of \citep{bennani2020generalisation}}}}_{2}^{2} \\
     &= \norm{\sum_{k=\tau_S+1}^{\tau_T}  U_{\tau_S}\Sigma_{\tau_S}V_{\tau_S}^{\top}V_{k}\Sigma_{k}U_{k}^{\top}[U_{k}\Sigma_{k}^{2}U_{k}^{\top}+\lambda I_{n_{k}}]^{-1}\tl{y}_{k}}_{2}^{2}  \quad (\text{SVD})\\
     &= \norm{ \sum_{k=\tau_S+1}^{\tau_T}  U_{\tau_S}\Sigma_{\tau_S}\underbrace{V_{\tau_S}^{\top}V_{k}}_{O^{\tau_S \rightarrow k}_{SGD}}\underbrace{\Sigma_{k}[\Sigma_{k}^{2}+\lambda I_{n_{k}}]^{-1}U_{k}^{\top} }_{M_{k}}\tl{y}_{k}}_{2}^{2} \\
\end{align}
Where we used the SVD $\phi(X^{\tau}) = U_{\tau} \Sigma_{\tau} V_{\tau}^T, \forall \tau \in [T]$. This concludes the proof. 
\end{proof}

\subsection{Proof of Corolary~\ref{corr:angle_forgetting_sgd}}  \label{proof:angle_forgetting_sgd}
For this proof, we will bound the Catastrophic Forgetting as a function of the principal angles between the source and target subspaces. Indeed, given two subspace $\tau_S$ and $\tau_T$ represented by their orthonormal basis concatenated respectively in $V_{\tau_S}$ and $V_{\tau_T}$, the elements of the diagonal matrix $\Theta^{\tau_S \rightarrow \tau_T}$ resulting from the SVD of $V_{\tau_S}^{\top}V_{\tau_T}$ are the cosines of the principal angles between these two subspace \citep{wedin1983angles,zhu2013angles}.
\begin{proof}

\begin{align}
    \Delta^{\tau_S \rightarrow \tau_T}(X^{\tau_S}) & \leq \sum_{k=\tau_S+1}^{\tau_T} \norm{ U_{\tau_S}\Sigma_{\tau_S}V_{\tau_S}^{\top}V_{k}\Sigma_{k}[\Sigma_{k}^{2}+\lambda I_{n_{k}}]^{-1}U_{k}^{\top} \tl{y}_{k}}_{2}^{2} \\
      & \leq \sum_{k=\tau_S+1}^{\tau_T} \norm{ U_{\tau_S} \Sigma_{\tau_S}}_{2}^{2} \norm{V_{\tau_S}^{\top}V_{k}}_{2}^{2} \norm{\Sigma_{k}[\Sigma_{k}^{2}+\lambda I_{n_{k}}]^{-1}U_{k}^{\top} \tl{y}_{k}}_{2}^{2} \quad (\text{sub-multiplicativity of norm 2}) \\
       & \leq \sum_{k=\tau_S+1}^{\tau_T} \norm{  \Sigma_{\tau_S}}_{2}^{2} \norm{V_{\tau_S}^{\top}V_{k}}_{2}^{2} \norm{M_{k} \tl{y}_{k}}_{2}^{2} \quad  \quad (U_{\tau_S} \text{ is an orthonormal matrix})\\
       & \leq \sigma_{\tau_S,1}^{2} \sum_{k=\tau_S+1}^{\tau_T}  \norm{Y \Theta^{\tau_S \rightarrow k}Z^{\top}}_{2}^{2} \norm{M_{k} \tl{y}_{k}}_{2}^{2} \quad  \quad (\text{SVD})\\
       & \leq \sigma_{\tau_S,1}^{2} \sum_{k=\tau_S+1}^{\tau_T}  \norm{ \Theta^{\tau_S \rightarrow k}}_{2}^{2} \norm{M_{k} \tl{y}_{k}}_{2}^{2} \quad  \quad ( Y,Z \text{ are orthonormal  matrices})\\
\end{align}
where $Y \Theta^{\tau_S \rightarrow k}Z^{\top}$ is the SVD of $V_{\tau_S}^{\top}V_{k}$. This concludes the proof.
\end{proof}

\subsection{Proof of Corollary~\ref{corr:cf_gem_nt}} \label{proof:gem_nt}
We first need to prove a corollary that is exactly the same as Corollary~\ref{cor:convergence_pca_finite_case} (shown below), the difference lies in the kernel definition.
Under the same notation as in Corollary~\ref{cor:convergence_pca_finite_case}, the solution after training on task $\tau$ for GEM-NT is such that:
% \guillaume{$\tilde{\phi}$ instead of $\bar{\phi}$ in the eq. below?}
\begin{align}
    \w_{\tau}^{\star}-\w_{\tau-1}^{\star}=\overline{\phi}_{\tau}(X^{\tau})^{\top}(\kappa_{\tau}(X^{\tau},X^{\tau})+\lambda I_{n_{\tau}})^{-1}\tilde{y}_{\tau}
\end{align}
where:
\begin{align*}
\kappa_{\tau}(x,x')&=\overline{\phi}_{\tau}(x)\overline{\phi}_{\tau}(x')^{\top}, \\
\overline{\phi}_{\tau}(x)&=\phi(x)T_{\tau-1}, \\
T_{\tau}&=I_{p}-G_{\tau}(G_{\tau})^{\top}, \\
\tilde{y}_{\tau}&=y_{\tau}-y_{\tau-1 \rightarrow \tau}, \\
y_{\tau-1 \rightarrow \tau}&=f_{\tau-1}^{\star}(X^{\tau}), \\
\end{align*}

\begin{proof} of Corollary~\ref{corr:cf_gem_nt}\
Similarly to Proof of Theorem~\ref{thm:forgetting_sgd}:
\begin{align}
 \Delta^{\tau_S \rightarrow \tau_T}(X^{\tau_S})&=\norm{\phi(X^{\tau_S})(\w_{\tau_T}^{\star}-\w_{\tau_S}^{\star})}_{2}^{2}  \quad (\text{Lemma~\ref{lemma:def_forgetting}}) \\
  &=\norm{\sum_{k=\tau_S+1}^{\tau_T}\phi(X^{\tau_S})(\w_{k}^{\star}-\w_{k-1}^{\star})}_{2}^{2} \\
   & = \norm{ \sum_{k=\tau_S+1}^{\tau_T} \phi(X^{\tau_S})\underbrace{\overline{\phi}(X^{k})^{\top}[\overline{\phi}(X^{k})\overline{\phi}(X^{k})^{\top}+\lambda I_{n_{k}}]^{-1}\overline{y}_{k}}_{\text{as shown above}}}_{2}^{2} \\
      &= \norm{\sum_{k=\tau_S+1}^{\tau_T}  U_{\tau_S}\Sigma_{\tau_S}V_{\tau_S}^{\top}T_{k-1}^{\top}V_{k}\Sigma_{k}U_{k}^{\top}[\overline{\phi}(X^{k})\overline{\phi}(X^{k})^{\top}+\lambda I_{n_{k}}]^{-1}\overline{y}_{k}}_{2}^{2} \quad (\text{SVD}) \\
        &= \norm{\sum_{k=\tau_S+1}^{\tau_T}  U_{\tau_S}\Sigma_{\tau_S}\underbrace{V_{\tau_S}^{\top}T_{k-1}T_{k-1}^{\top}V_{k}}_{O^{\tau_S \rightarrow k}_{\text{GEM-NT}}}\underbrace{\Sigma_{k}U_{k}^{\top}[\overline{\phi}(X^{k})\overline{\phi}(X^{k})^{\top}+\lambda I_{n_{k}}]^{-1}}_{M_k}\overline{y}_{k}}_{2}^{2}  \quad (T_{k-1})^{n}=T_{k-1} , \forall n \geq 1 \\
\end{align}
This concludes the proof.
\end{proof}

\subsection{Forgetting for \algoname} \label{proof:forgetting_pca}
To prove the forgetting expression for \algoname, we will use a corollary arising naturally from Theorem 1 of \citep{bennani2020generalisation} which extends the expression of the learned weights $(\w_{\tau+1}^{\star}-\w_{\tau}^{\star})$ from the \textbf{infinite} to the \textbf{finite} memory case. The proof will be shown after the proof of Corollary~\ref{corr:forgetting_pca} for the flow of the understanding.

\begin{corr} [Convergence of \algoname under finite memory] \label{cor:convergence_pca_finite_case} \  

Given $\cT_{1},...,\cT_{T}$ a sequence of tasks.
If the learning rate satisfies: $\eta_{\tau}< \frac{1}{\norm{\kappa_{\tau}(X^{\tau},X^{\tau}) +\lambda I_{n_{\tau}}}} $, $\kappa_\tau, \forall \tau \in [T]$ is invertible with a weight decay regularizer  $\lambda>0$, the solution after training on task $\tau$ is such that: \\
\begin{equation} \label{eq:update_pca_ogd}
\w_{\tau}^{\star}-\w_{\tau-1}^{\star}=\tilde{\phi}_{\tau}(X^{\tau})^{\top}(\kappa_{\tau}(X^{\tau},X^{\tau})+\lambda I_{n_{\tau}})^{-1}\tilde{y}_{\tau}
\end{equation}
where:
% \begin{equation*}
\begin{align*}
\kappa_{\tau}(x,x')&=\tilde{\phi}_{\tau}(x)\tilde{\phi}_{\tau}(x')^{\top}, \\
\tilde{\phi}_{\tau}(x)&=\phi(x)T_{\tau-1,:d}, \\
T_{\tau,:d}&=I_{p}-P_{\tau,:d}P_{\tau,:d}^{\top}, \\
\phi(x)&=\nabla_{\omega_{0}}f_{0}^{\star}(x), \\
\tilde{y}_{\tau}&=y_{\tau}-y_{\tau-1 \rightarrow \tau}, \\
y_{\tau-1 \rightarrow \tau}&=f_{\tau-1}^{\star}(X^{\tau}), \\
\end{align*}
where $T_{0,:d}=I_{p}$ since there are no previous task when training on task $1$.
\end{corr}

\begin{proof} of Corollary~\ref{corr:forgetting_pca}\

Similarly to Proof of Theorem~\ref{thm:forgetting_sgd}:
\begin{align}
    \Delta^{\tau_S \rightarrow \tau_T}(X^{\tau_S})&=\norm{\phi(X^{\tau_S})(\w_{\tau_T}^{\star}-\w_{\tau_S}^{\star})}_{2}^{2}  \quad (\text{Lemma~\ref{lemma:def_forgetting}}) \\
    &=\norm{\sum_{k=\tau_S+1}^{\tau_T}\phi(X^{\tau_S})(\w_{k}^{\star}-\w_{k-1}^{\star})}_{2}^{2} \\
     & = \norm{ \sum_{k=\tau_S+1}^{\tau_T} \phi(X^{\tau_S})\underbrace{\tl{\phi}(X^{k})^{\top}[\tl{\phi}(X^{k})\tl{\phi}(X^{k})^{\top}+\lambda I_{n_{k}}]^{-1}\tl{y}_{k}}_{\text{from Corollary~\ref{cor:convergence_pca_finite_case} }}}_{2}^{2} \\
     &= \norm{\sum_{k=\tau_S+1}^{\tau_T}  U_{\tau_S}\Sigma_{\tau_S}V_{\tau_S}^{\top}T_{k-1,:d}^{\top}V_{k}\Sigma_{k}U_{k}^{\top}[\tl{\phi}(X^{k})\tl{\phi}(X^{k})^{\top}+\lambda I_{n_{k}}]^{-1}\tl{y}_{k}}_{2}^{2}  \quad (\text{SVD})\\
      &= \norm{\sum_{k=\tau_S+1}^{\tau_T}  U_{\tau_S}\Sigma_{\tau_S}\underbrace{V_{\tau_S}^{\top}R_{k-1,d:}R_{k-1,d:}^{\top}V_{k}}_{O^{\tau_S \rightarrow k}_{PCA}}\Sigma_{k}\underbrace{U_{k}^{\top}[\tl{\phi}(X^{k})\tl{\phi}(X^{k})^{\top}+\lambda I_{n_{k}}]^{-1}}_{M_k}\tl{y}_{k}}_{2}^{2}  
    %  &= \norm{\sum_{k=\tau_S+1}^{\tau_T}  U_{\tau_S,d:}\Sigma_{\tau_S,d:}\underbrace{V_{\tau_S}^{\top}R_{k-1,d:}R_{k-1,d:}^{\top}V_{k}}_{O^{\tau_S \rightarrow k}_{PCA}}\Sigma_{k}\underbrace{U_{k}^{\top}[\tl{\phi}(X^{k})\tl{\phi}(X^{k})^{\top}+\lambda I_{n_{k}}]^{-1}}_{M_k}\tl{y}_{k}}_{2}^{2}  
\end{align}   

\end{proof}

\begin{proof}[Proof of Corollary~\ref{cor:convergence_pca_finite_case}]
    \begin{align*}
            \intertext{In the same fashion as \citep{bennani2020generalisation}, we prove Corollary \ref{cor:convergence_pca_finite_case} by induction. Our induction hypothesis $H_\tau$ is the
    following : $\cH_\tau$ : For all $k \leq \tau$, Corollary \ref{cor:convergence_pca_finite_case} holds.}
            \intertext{First, we prove that $\cH_1$ holds.}
            \intertext{The proof is straightforward. For the first task, since there were no previous tasks,
            \algoname on this task is the same as SGD.}
            \intertext{Therefore, it is equivalent to minimising the following objective : }
            \argmin_{\w \in \mathbb{R}^{p}} &\norm{f_{0}(X^{1})+\phi(X^{1})(\w - \w_{0}^{\star}) -
            y_{1}}_{2}^{2} +\frac{1}{2}\lambda \norm{\w-\w_{0}}_{2}^{2}
            \intertext{where $\phi(x) = \nabla_{\w_{0}^\star} f_{0}^{\star}(x)$.}
            \intertext{Substituing the residual term $\tl{y}_{1}=y_{1}-f_{0}(X^{1})$, we get:}
            \argmin_{\w \in \mathbb{R}^{p}} &\norm{\phi(X^{1})(\w - \w_{0}^{\star}) -
            \tl{y}_{1}}_{2}^{2} +\frac{1}{2}\lambda \norm{\w-\w_{0}}_{2}^{2}
            \intertext{The objective is quadratic and the Hessian is positive definite, therefore the minimum exists and is unique}
%             :}
            \w_{1}^{\star} - \w_{0}^{\star} &= \phi(X^{1})^{\top} (\phi(X^{1}) \phi(X^{1})^{\top}+\lambda I_{n_{1}})^{-1}\tl{y}_{1} &\\
            \intertext{Under the NTK regime assumption :}
            f_{1}^{\star}(x) &= f_{0}^{\star}(x) + \nabla_{\w_{0}} f_{0}^{\star}(x)^{\top} (\w_{1}^{\star} - \w_{0}^{\star})
            \intertext{Then, by replacing into $\w_{1}^{\star} - \w_{0}^{\star}$ :}
            f_{1}^{\star}(x) &= f_{0}^{\star}(x) + \nabla_{\w_{0}} f_{0}^{\star}(x) \phi(X^{1})^{\top} (\phi(X^{1}) \phi(X^{1})^{\top}+\lambda I_{n_{1}})
            ^{-1}\tl{y}_{1} &\\
            f_{1}^{\star}(x) &= f_{0}^{\star}(x) + \kappa_1(x,X^{1})  (\kappa_1(X^{1},X^{1})+\lambda I_{n_{1}})^{-1}
            \tl{y}_{1} &\\
            \intertext{Finally :}
            f_{1}^{\star}(x) - f_{0}^{\star}(x) &= \kappa_1(x,X^{1})  ( \kappa_1(X^{1},X^{1})+\lambda I_{n_{1}})^{-1}
            \tl{y}_{1} &\\
            \intertext{Where :} 
            \kappa_{1}(X^{1},X^{1})&=\tl{\phi_{1}}(X^{1})\tl{\phi}_{1}(X^{1})^{\top} \\
            &=\phi(X^{1})T_{0,:d}T_{0,:d}^{\top}\phi(X^{1})^{\top} \\
            &=\phi(X^{1})\phi(X^{1})^{\top} \\
            \intertext{Since there is no previous task and $T_{0,:d}$ contains no eigenvectors yet, we have $T_{0,:d}=I_{p}$ and $\tilde{y}_{1}=y_{1}$.}
            \intertext{This completes the proof of $\cH_1$.}
    \intertext{Let $\tau \in \cN^{\star}$, we assume that $\cH_\tau$ is true, then we show that $\cH_{\tau+1}$ is true.}
    \intertext{At the end of training of task $\tau$, we add the first $d$ eigenvectors of $\phi(X^{\tau})\phi(X^{\tau})^{\top}$ to $P_{\tau-1,:d} \in \mathbb{R}^{p \times (\tau-1)\cdot d}$ to form the matrix $P_{\tau,:d} \in \mathbb{R}^{p \times \tau \cdot d}$ through PCA decomposition}.
    \intertext{The update during the training of task $\tau+1$ is projected orthogonally to the first $d$ components of task $1$ until $\tau$ via the matrix $T_{\tau,:d}$:}
    \w_{\tau+1}(t+1)&=\w_{\tau}^{\star}-\eta T_{\tau,:d}\nabla_{\w}\cL^{\tau}_{\lambda}(\w_{\tau+1}(t)) \\
    \w_{\tau+1}(t+1)-\w_{\tau}^{\star}&=-\eta T_{\tau,:d}\nabla_{\w}\cL^{\tau}_{\lambda}(\w_{\tau+1}(t)) \\
    \w_{\tau+1}(t+1)-\w_{\tau}^{\star}&= T_{\tau,:d} \tl{\w}_{\tau+1}
    \intertext{Where $\eta$ is the learning rate and $T_{\tau,:d}=I_{p}-P_{\tau,:d}P_{\tau,:d}^{\top}$.}
    \intertext{We rewrite the objective by plugging in the variables we just defined. The two objectives are
        equivalent : }
        \argmin_{\tl{\w}_{\tau+1}  \in \mathbb{R}^{p}} &\norm{\underbrace{\phi(X^{\tau+1}) T_{\tau,:d}}_{\phi_{\tau+1}(X^{\tau+1})} \tl{w}_{\tau+1} -
            \tl{y}_{\tau+1}}_2^2  &\\
    \intertext{The optimisation objective is quadratic, unconstrainted, with a positive definite hessian.
            Therefore, an optimum exists and is unique :  }
        \tl{\w}^{\star}_{\tau+1} &= \phi_{\tau+1}(X^{\tau+1})^{\top} (\phi_{\tau+1}(X^{\tau+1}) \phi_{\tau+1}(X^{\tau+1})^{\top} )^{-1}
            \tl{y}_{\tau+1} &\\
             \w_{\tau+1}^{\star}-\w_{\tau}^{\star} &= \phi_{\tau+1}(X^{\tau+1})^{\top} (\phi_{\tau+1}(X^{\tau+1}) \phi_{\tau+1}(X^{\tau+1})^{\top} )^{-1}
            \tl{y}_{\tau+1} &\\
                 \w_{\tau+1}^{\star}-\w_{\tau}^{\star} &= \phi_{\tau+1}(X^{\tau+1})^{\top} (\kappa_{\tau+1}(X^{\tau+1},X^{\tau+1}) )^{-1}\tl{y}_{\tau+1} &\\
        \intertext{Recall from the induction hypothesis of $\cH_\tau$ the general form of $f_{\tau}^{\star}(x)$ :}
         f_{\tau}^{\star}(x) &= f_{\tau-1}^{\star}(x) + \langle \nabla_{\w_{0}} f_{0}^{\star}(x) ,\w_{\tau+1}^{\star}
             - \w_{\tau}^{\star} \rangle &\\
             \intertext{After training on task $\tau+1$ :}
              f_{\tau+1}^{\star}(x) &= f_{\tau-1}^{\star}(x) + \langle \nabla_{\w_{0}} f_{0}^{\star}(x) ,\w_{\tau+1}^{\star}
             - \w_{\tau-1}^{\star} \rangle &\\
              f_{\tau+1}^{\star}(x) &= f_{\tau-1}^{\star}(x) + \langle \nabla_{\w_{0}} f_{0}^{\star}(x) ,\w_{\tau+1}^{\star}
             - \w_{\tau-1}^{\star} + \underbrace{\w_{\tau}^{\star}-\w_{\tau}^{\star}}_{=0} \rangle &\\
             f_{\tau+1}^{\star}(x) &= \underbrace{f_{\tau-1}^{\star}(x) + \langle \nabla_{\w_{0}} f_{0}^{\star}(x) ,\w_{\tau}^{\star}-\w_{\tau-1}^{\star} \rangle}_{f_{\tau}^{\star}(x)} 
             + \langle \nabla_{\w_{0}} f_{0}^{\star}(x) ,\w_{\tau+1}^{\star}
             - \w_{\tau}^{\star} \rangle &\\
              f_{\tau+1}^{\star}(x) &= f_{\tau}^{\star}(x) + \langle \nabla_{\w_{0}} f_{0}^{\star}(x) ,\w_{\tau+1}^{\star}
             - \w_{\tau}^{\star} \rangle &\\
             f_{\tau+1}^{\star}(x) &= f_{\tau}^{\star}(x) +  \phi(x) \phi_{\tau+1}(X^{\tau+1})^{\top} (\kappa_{\tau+1}(X^{\tau+1},X^{\tau+1}) )^{-1}\tl{y}_{\tau+1}  &\\
             f_{\tau+1}^{\star}(x) &= f_{\tau}^{\star}(x) +  \phi(x) T_{\tau,:d}^{\top}\phi_{\tau+1}(X^{\tau+1})^{\top} (\kappa_{\tau+1}(X^{\tau+1},X^{\tau+1}) )^{-1}\tl{y}_{\tau+1}  &\\
               f_{\tau+1}^{\star}(x) &= f_{\tau}^{\star}(x) +  \underbrace{\phi(x) T_{\tau,:d}T_{\tau,:d}^{\top}\phi_{\tau+1}(X^{\tau+1})^{\top}}_{\kappa_{\tau+1}(x,X^{\tau+1})} (\kappa_{\tau+1}(X^{\tau+1},X^{\tau+1}) )^{-1}\tl{y}_{\tau+1}  \quad (\text{ since }   (T_{\tau,:d})^{\top}=T_{\tau,:d} )\\
               f_{\tau+1}^{\star}(x) &= f_{\tau}^{\star}(x) +  \kappa_{\tau+1}(x,X^{\tau+1})(\kappa_{\tau+1}(X^{\tau+1},X^{\tau+1}) )^{-1}\tl{y}_{\tau+1}  \\
                \intertext{We have proven $\cH_{t+1}$ and conclude the proof of Corollary \ref{cor:convergence_pca_finite_case}. }
        \end{align*}
\end{proof}

\subsection{Algorithms summary}

\begin{table}[h] 
    \centering
\begin{tabular}{c|c |c|c |c}
\hline\hline
\multirow{2}{*}{Properties}     & $O^{\tau_S \rightarrow \tau_T}_{X}$           &  $X$  & Elements  stored  & Recompute   \\ 
  & $=V_{\tau_S}^{\top}\mathbf{XX}^{\top}V_{\tau_T}$  &   contains                          & in the memory  &  NTK?  \\ \hline
SGD         &     X=$I_{\tau_S}$    & NA                                      & NA  & NA\\
GEM-NT        &    X=\textcolor{greencustom}{$\mathbf{\overline{G}_{\tau_T-1}}$}                       & samples of $\nabla\cL(X^{\tau})$  &  samples of $X^{\tau}$  & Yes \\ 
% AGEM         &           &  samples of $X^{\tau}$                  & $\nabla\cL(X^{\tau})$ & Yes     \\
OGD          &  X=\textcolor{redcustom}{$\mathbf{R_{\tau_T-1}}$}                   & samples of $\nabla f(X^{\tau})$ &  samples of $\nabla f(X^{\tau})$   & No  \\ 
\algoname    &   X=\textcolor{bluecustom}{$\mathbf{R_{\tau_T-1,:d}}$}             &   top eigenvectors of $\nabla f(X^{\tau})$ &  top  eigenvectors of $\nabla f(X^{\tau})$ & No \\ \hline
\end{tabular}
    \caption{Property of the Overlap matrix for each method which is responsible for mitigating Catastrophic Forgetting. NA: Not applicable} \label{tab:summary_alg}
\end{table}

\paragraph[]{NTK overlap matrix}
First of all, the three methods GEM-NT, OGD, \algoname differs from SGD by the the matrix $X$ (1st column) that contains either the features map $\nabla_{\w}f(x)$ or the gradient loss function.

\paragraph[]{Elements stored}
GEM-NT and OGD both samples random elements at the end of each task $\tau$ to store in the memory. 
For the sake of understanding, if we assume a mean square loss function, with a batch size equal to one, the gradient loss function becomes:  
$g_{\tau}^{\text{(GEM-NT)}}=\underbrace{\nabla_{\w}f_\tau(x)}_{g_{\tau}^{(\text{OGD})}}(f_\tau(x)-y_\tau)$. From here, we see that GEM-NT weights the features maps by the residual of a given task $k < \tau$ when training on task $\tau $.

\paragraph[]{Information compression}
\algoname compresses the information contained in the data by storing the principal components of $\nabla_{\w}f(X^{\tau})$ through PCA. If the data has structure such as Rotated MNIST or Split CIFAR (See 
Section~\ref{fig:dataset_explained_variance}), storing few components will explain a high percentage of variance of the data of component in order to 
explain the dataset variance.

\paragraph[]{Accounting for the NTK variation}
The drawback OGD and \algoname compared to GEM-NT is that the NTK is assumed to be constant which is not always the case in practice (See Section~\ref{sec:ntk_variation}). \algoname and OGD will then project orthogonally to a vector that is outdated.
\newpage

\subsection{The counter-example of Permuted MNIST: no structure} \label{sec:permuted_mnist}

We now examine the dataset Permuted MNIST and try to understand why \algoname is not efficient in such case. Each task is an MNIST dataset where a different and uniform permutation of pixels is applied. This has the particularity of removing any extra-task correlations and patterns. 

\subparagraph{Eigenvalues of the NTK overlap matrix :}
Figure~\ref{fig:permuted_comparison_wrt_buffer_size2} shows the eigenvalues of the NTK overlap matrix when increasing buffer size. First of all, we notice that the magnitude of the eigenvalues is very small compared to Figure~\ref{fig:rotated_comparison_wrt_buffer_size}. This is explained by the fact that each task shares almost no correlations, meaning that the cosine of angle of the two subspaces might be close to 0 (small eigenvalues). Additionally, we see that increasing the memory size does not reduce much the eigenvalue magnitude. This is due to the distribution of eigenvalues (See Figure~\ref{fig:dataset_explained_variance}) which are spread more uniformly than Permuted MNIST and Split CIFAR, meaning that more components need to be kept in order to explain a high $\%$ of the variance. In this situation, \algoname does not have much advantage compared to OGD (See also toy example, Supplementary Material Section \ref{toy:worst_case_pca}).

\subparagraph{Final accuracy with OGD :}
We now compare final performance against OGD (See Figure~\ref{fig:permuted_bar2}). \algoname does sensitively well compared to OGD (except for the first task where performance are much worse). This can be explained by the fact that \algoname needs to keep a lot of components to explain a high percentage of variance such that selecting random element like OGD will results in comparable results. This is all the more confirmed by Table~\ref{table:dataset_explained_variance}, keeping $50$ components only explains $50\%$ of the variance while it respectively explains $81\%$ for Rotated MNIST and $72\%$ for Split CIFAR. As mentionned by \citep{farquhar2018towards}, even though such datasets meets the definition of CL, it is an irrealistic setting since 
 ``new situations look confusingly similar to old ones'\,''. Hence methods that leverage structure like \algoname can be useful.

\begin{figure}[h!]
\hspace{-2mm}
    \subfloat[
    Comparison of the eigenvalues of $O^{1 \rightarrow 2}$ on  \\
    \textbf{Permuted MNIST} with increasing memory size.  \\
    Lower values imply less forgetting.]{%
      \includegraphics[width=0.50\textwidth]{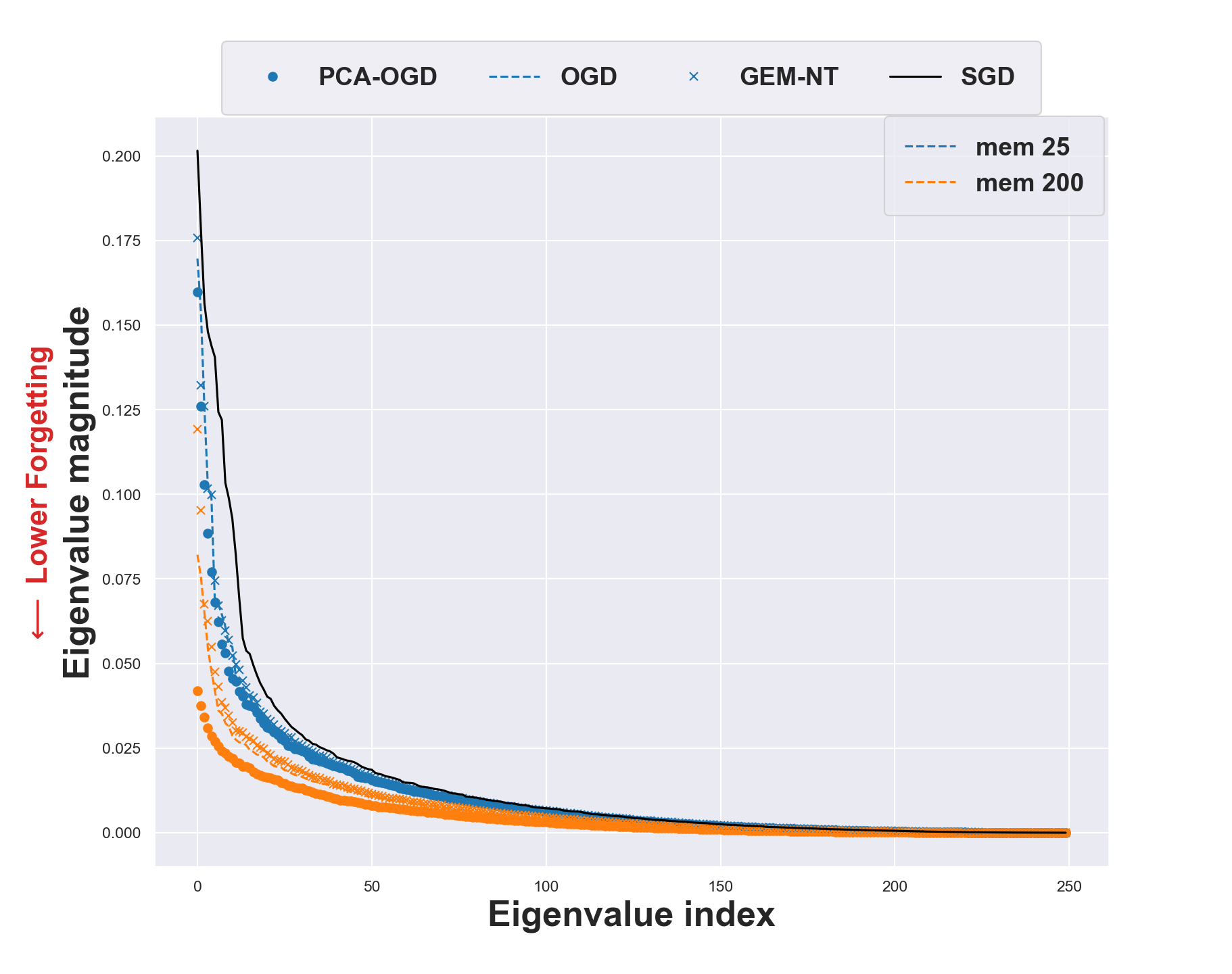}
      \label{fig:permuted_comparison_wrt_buffer_size2}
    }
\hspace{-2mm} 
    \subfloat[Final accuracy on \textbf{Permuted MNIST} for different memory size (averaged over $5$ seeds $\pm 1$ std). OGD and \algoname have comparable performance (except for the first task).]{%
      \includegraphics[width=0.50\textwidth,height=0.45\textwidth]{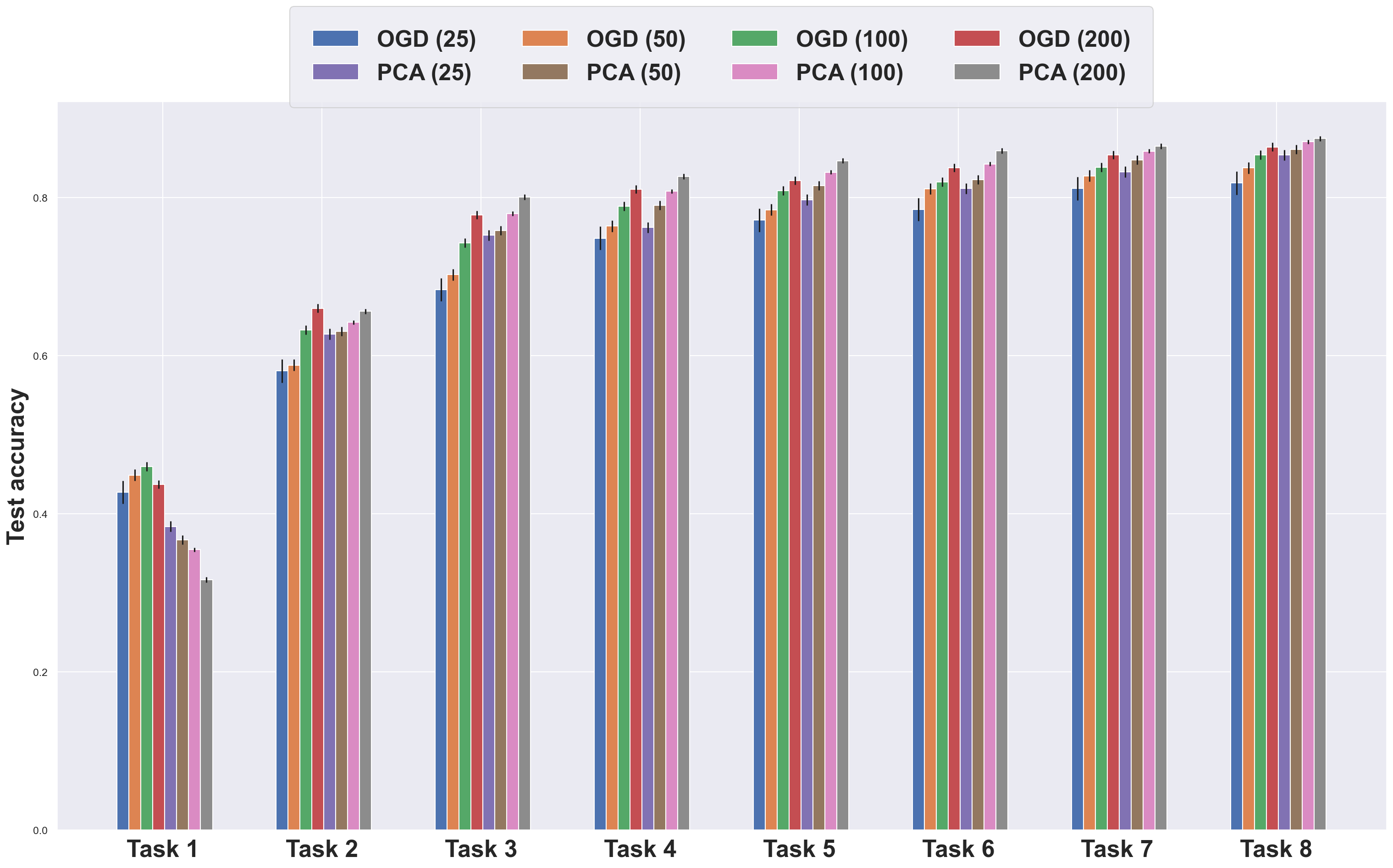}
      \label{fig:permuted_bar2}
    }
\end{figure}
% \begin{figure}[ht]
% \hspace{-2.0em}
%     % \centering
%     \includegraphics[width=1.2\linewidth]{figures/permuted_comparison_similarity_wrt_mem_size_task_2.png}
%     \caption{Comparison of the eigenvalues of $O^{1 \rightarrow 2}$ on \textbf{Permuted MNIST} with increasing memory size. Lower values imply less forgetting.}
%     \label{fig:permuted_comparison_wrt_buffer_size}
% \end{figure}

\newpage

\subsection{Comparison \algoname versus OGD}

\begin{figure*}[h!]
    \centering
    \includegraphics[width=1\linewidth,height=0.45\linewidth]{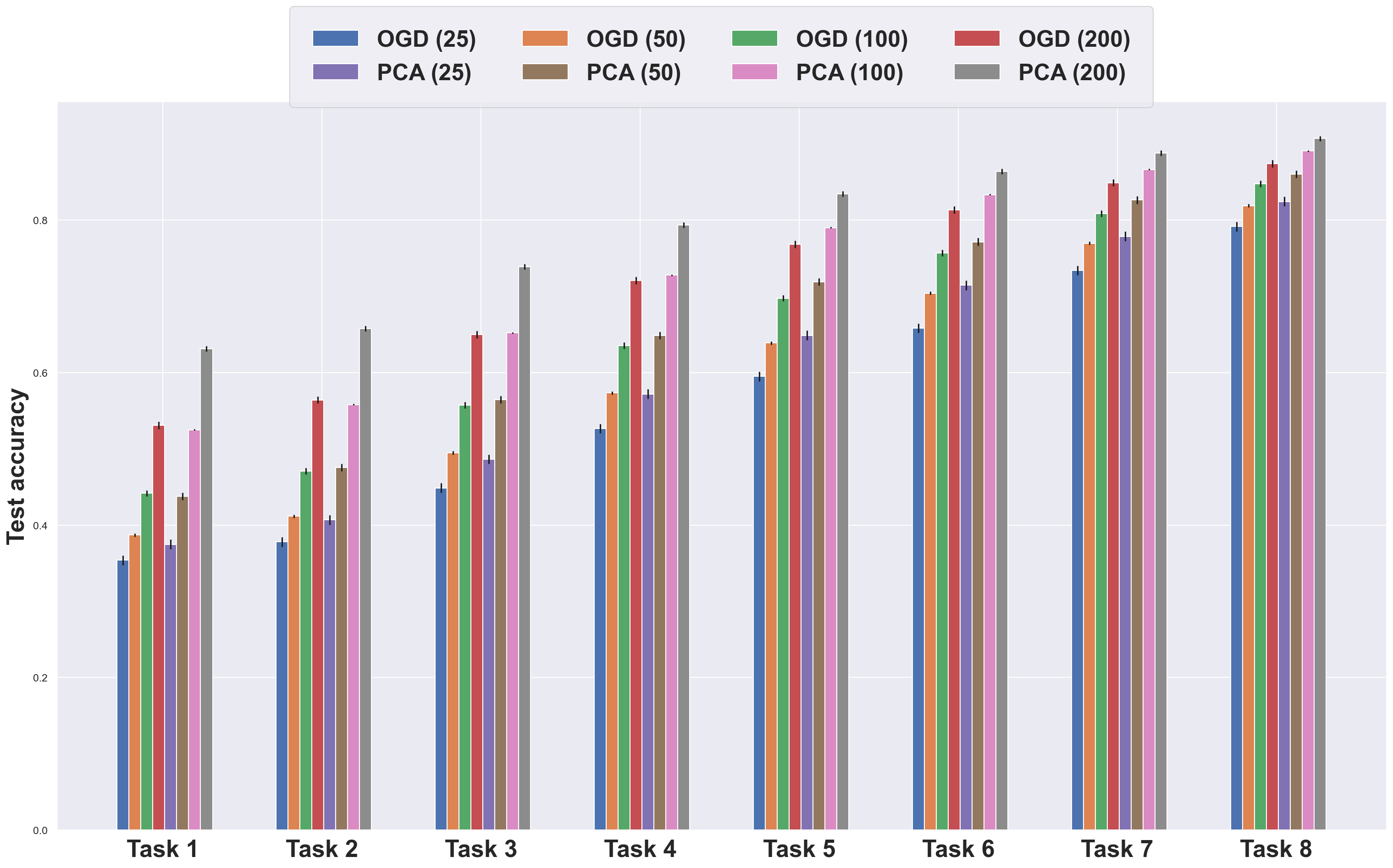}
    \centering
    \caption{Final accuracy on \textbf{Rotated MNIST} for different memory size (averaged over $5$ seeds $\pm 1$ std). OGD needs twice as much memory as \algoname in order to achieve the same performance (i.e compare OGD (200) and PCA (100).}
    \label{fig:rotated_bar}
\end{figure*}

\begin{figure*}[h!]
    \centering
    \includegraphics[width=1\linewidth,height=0.45\linewidth]{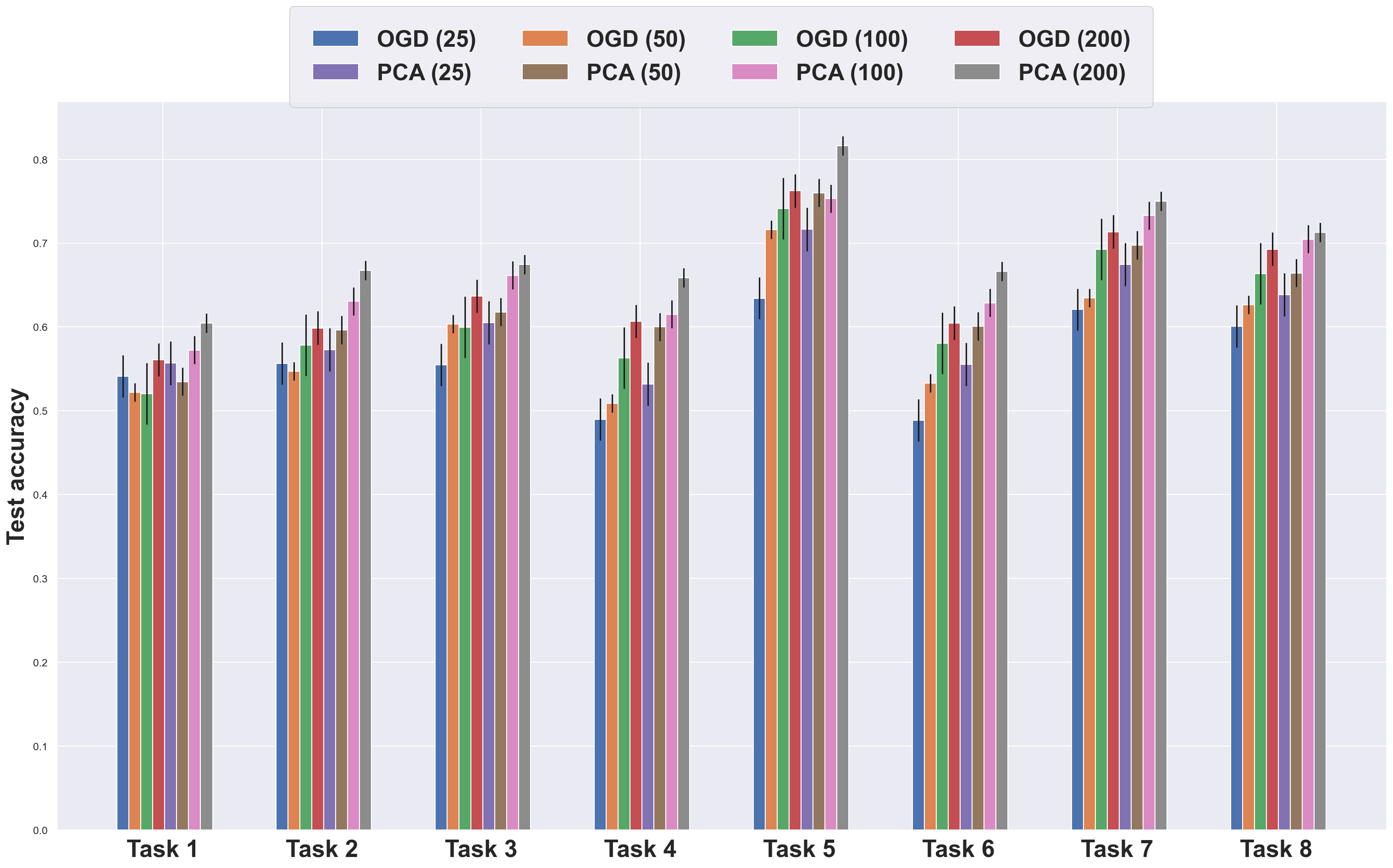}
    \centering
    \caption{Final accuracy on \textbf{Split CIFAR} for different memory size (averaged over $5$ seeds $\pm 1$ std). When dataset is well structured \algoname efficiently leverages the pattern (i.e compare OGD (200) and PCA (100).}
     \label{fig:split_cifar_bar}
\end{figure*}

\begin{figure*}[h!]
    \centering
    \includegraphics[width=1\linewidth,height=0.45\linewidth]{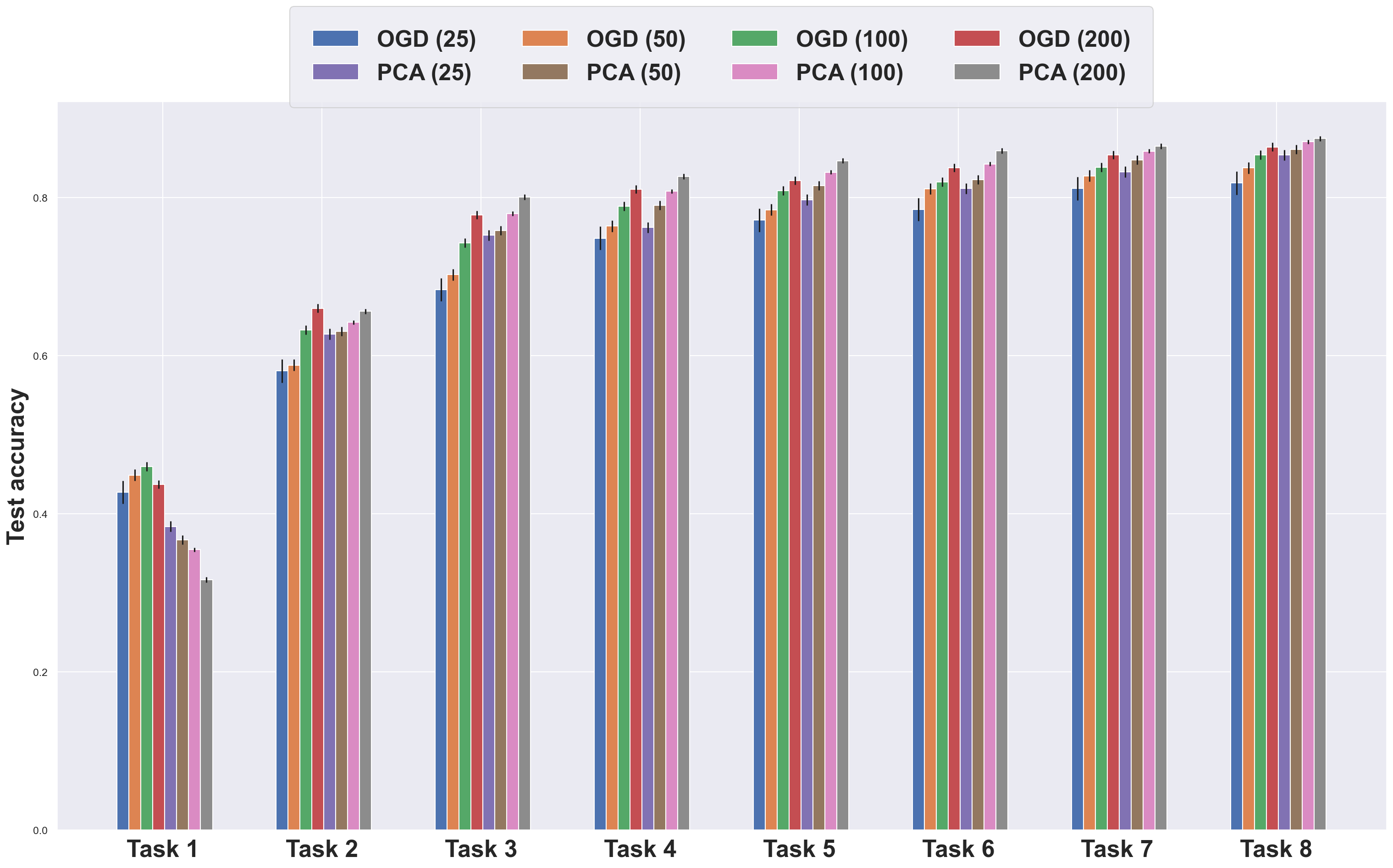}
    \centering
    \caption{Final accuracy on \textbf{Permuted MNIST} for different memory size. When there is no structure, information captured by \algoname from previous tasks cannot be leveraged for future tasks. OGD and \algoname have comparable performance (except for the first task).}
    \label{fig:permuted_bar}
\end{figure*}

\newpage

\subsection{Structure in the data}
We sample a subset of $s=3,000$ samples from different datasets $x^{\tau}_{j}$,$j=1,...,3000$ (Permuted and Rotated MNIST), then we perform PCA on $\phi(x^{\tau}_{j})\phi(x^{\tau}_{j})^{\top}$ and keep the $d$ top components. Having a total memory size of $M=200,500,1000,2000,3000$ and training on $15$ tasks means that each task will be allocated $M/14$ since we omit the last task. As seen in Figure \ref{fig:dataset_explained_variance} for a total memory size of $200$, we only keep $14$ components which corresponds to $38.84 \%$ of the variance explained in Permuted MNIST while it represents $71.06 \%$ for Rotated MNIST. This is naturally explained by the fact that having random permutations breaks the structure of the data and in order to keep the most information would we need to allocate a large amount of memory.

\begin{figure}[h!]
    \centering
    \includegraphics[width=1\linewidth]{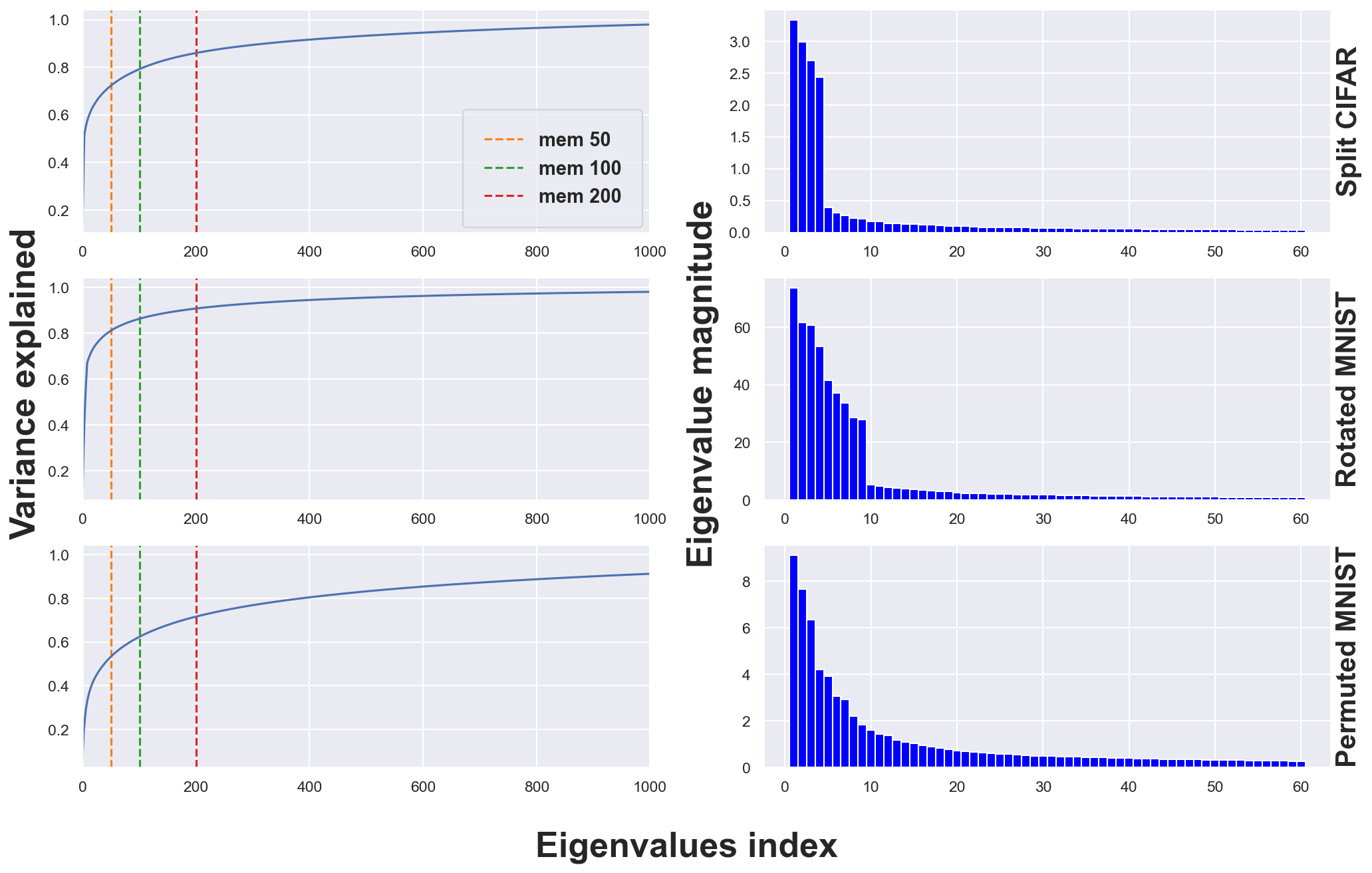}
    \caption{Percentage of variance explained for different datasets. Verticale lines on the left represents the number of components kept or the memory allocated per task. We have truncated the x-axis to focus on the interesting part.}
    \label{fig:dataset_explained_variance}
\end{figure}

\begin{table}[h!]
    \centering
\begin{tabular}{c|c c c}
\hline\hline
components kept     & Permuted MNIST & Rotated MNIST & Split CIFAR  \\ \hline
10         & 35.13  &  68.52 &  58.87\\
25         & 45.14  &  75.54 &  65.99\\
50         & 53.33  &  81.09 &  72.27\\
100             & 62.37 & 86.23  &  79.19  \\ 
200             & 71.65 & 90.71 &    85.99 \\
500            & 83.20 & 94.42   &  93.24 \\ \hline
\end{tabular}
    \caption{Percentage of variance explained with different memory size when performing PCA on $s=3,000$ samples (except for Split CIFAR where $s=1,500$).}
    \label{table:dataset_explained_variance}
\end{table}

% \begin{figure}[ht]
%     \centering
%     \includegraphics[width=0.6\linewidth]{figures/uniform_egeinvalue.png}
%     \caption{The first eigenvalues are sensitively of the same magnitude order (left) such that taking the first $25 (5\%)$ only explains $26 \%$ of the data variance (right).}
%     \label{fig:eigenvalues_worst_case}
% \end{figure}
% \thang{j'imagine la figure 10 est plus interessant que la 11? je vais p-e combiner figure 9 et 10 ensemble}
% \begin{figure}[ht]
%     \centering
%     \includegraphics[width=0.6\linewidth]{figures/training_loss_toy.png}
%     \caption{Testing loss of OGD (left) versus PCA-OGD (right). We see that OGD incurs almost no forgetting while PCA-OGD has drastic variation in the testing loss over the time.}
% \end{figure}

% \begin{figure}[ht]
%     \centering
%     \includegraphics[width=0.6\linewidth]{figures/forgetting_error_toy.png}
%     \caption{Forgetting error of OGD (left) versus PCA-OGD (right).}
% \end{figure}

\newpage

 \newpage
 
\subsection{NTK changes} \label{sec:ntk_variation}
 We measure the change in NTK of \algoname from its initialization value for different dataset size for a fixed architecture after each task (See Figure~\ref{fig:ntk_variation}). The green curve  shows the actual parameters used for the experiments. Although, there is linear increase of the NTK for MNIST datasets, it is approximately constant (after the first task) for Split CIFAR which validates the constant NTK assumption and explains the good result of \algoname  for this dataset.
\begin{figure}[h!]
     \centering
    \includegraphics[width=1.1\linewidth]{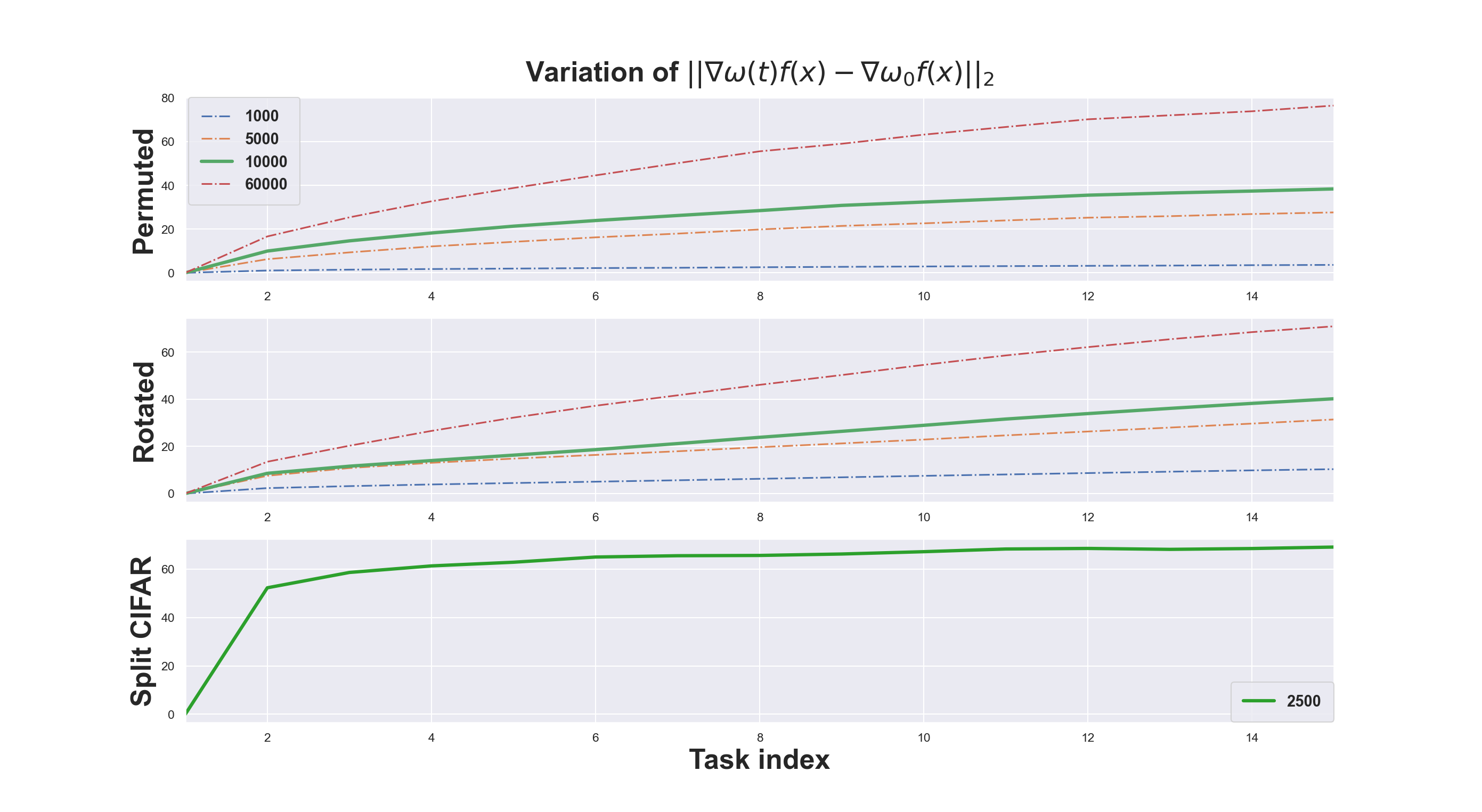}
    \caption{Variation of the NTK for different datasets size (legend).} \label{fig:ntk_variation}
\end{figure}

\newpage

\subsection{Experimental setup and general performance} \label{app:experiments_details}

\textbf{Datasets} We are considering four datasets \textbf{Permuted MNIST} \citep{ewc}, \textbf{Rotated MNIST} \citep{gem}, \textbf{Split MNIST} and \textbf{Split CIFAR} \citep{zenke2017continual}. For MNIST dataset, we sampled $1,000$ examples from each task leading to a total training set size of $10,000$ as in \citep{gem,aljundi2019online}.
\begin{itemize}
    \item Permuted MNIST is coming from the 0-9 digit dataset MNIST \citep{lecun1998gradient} where each pixels have been permuted randomly. Each task corresponds to a new permutation randomly generated (but fixed along all the dataset samples).
    \item Rotated MNIST is the same MNIST dataset where each new task corresponds to a fixed rotation of each digit by a fixed angle. Our $15$ tasks correspond to a fixed rotation of $5$ degres with respect to the previous task.
    \item Split MNIST consists of $5$ binary classification tasks where we split the digit such as: $0/1$ , $2/3$ , $4/5$ , $6/7$ , $8/9$.
    \item Split CIFAR comes from CIFAR-100 dataset \citep{krizhevsky2009learning} which contains $100$ classes that can be grouped again into $20$ superclasses. Split CIFAR-100 \citep{gem} is constructed by splitting the dataset into $20$ disjoincts classes sampled without replacement. The $20$ tasks are then composed of $5$ classes.
\end{itemize}

\textbf{Baselines}  We are comparing our method \algoname along with SGD, A-GEM \citep{agem} and OGD \citep{farajtabar2020orthogonal}. 

\textbf{Optimizer} We use Stochastic Gradient for each method and grid search to find hyperparameters that gave best results: learning rate of $1e-3$, batch size of $32$ and $10$ epochs for each tasks.

\textbf{Performance Metrics} Following \citep{agem} we report the Average Accuracy $A_{T}$ and the Forgetting Measure $F_{T}$:
\begin{align*}
    A_{T}=\frac{1}{T}\sum_{\tau=1}^{T}a_{T,\tau}
\end{align*}
Where $a_{\tau,T}$ represents the accuracy of task $\tau$ at the end of the training of task $T \geq \tau$.

\textbf{Forgetting Measure} \citep{gem} used the average forgetting as the performance drop of task $\tau$ over the training of later tasks:
\begin{align*}
    F_{T}=\frac{1}{T-1}\sum_{\tau=1}^{T-1}f_{\tau}^{T}
\end{align*}
where $f_{\tau}^{T}$ is defined as the highest forgetting from task $\tau$ until $T$:
\begin{align*}
    f_{\tau}^{T}=\max_{l=\tau,...,T} a_{l,\tau}-a_{T,\tau}
\end{align*}

\begin{table}[h!]
    \centering
\begin{tabular}{c|c |c| c}
\hline\hline
Hyperparameters    & Split MNIST & Rotated/Permuted MNIST & CIFAR-100  \\ \hline
Dataset size (per task)   & 2,000     & 10,000  &  2,500  \\
Epochs              & 5    & 10 &     50  \\ 
Architecture          & MLP  & MLP & LeNet$^1$      \\
Hidden dimension          & 100   & 100 & 200      \\
EWC regularizer &  10  & 10 & 25 \\
$\#$ tasks               & 5 & 15 & 20      \\ \hline
Optimizer      & \multicolumn{3}{c}{SGD}   \\ 
Learning rate      & \multicolumn{3}{c}{1e-03}   \\ 
Batch size        & \multicolumn{3}{c}{32}    \\
Torch seeds       & \multicolumn{3}{c}{0 to 4} \\ 
Memory size  & \multicolumn{3}{c}{100}    \\
PCA sample size $s$  & \multicolumn{3}{c}{3,000}    \\
samples  used to compute $O^{\tau_S \rightarrow \tau_T}$ (per task)   & \multicolumn{3}{c}{250}    \\\hline
\end{tabular}
    \caption{Hyperparameters used across experiments.
   }
\end{table}

\footnotetext[1]{ For this architecture, we observed better results when only projecting orthogonally to the fully connected layers.}

\newpage 

\begin{table}[h!]
    \centering
% \begin{adjustbox}{center, height=1.4\textheight}
\resizebox{1.0\columnwidth}{!}{
\begin{tabular}{c |  c  c  c  c  c  c  c  } 
\hline\hline
 &  \textbf{Task 1}  &  \textbf{Task 2}  &  \textbf{Task 3}  &  \textbf{Task 4}  &  \textbf{Task 5}  &  \textbf{Task 6}  &  \textbf{Task 7}   \\ \hline
 SGD & 25.55 $\pm$ 0.99 & 27.79 $\pm$ 0.85 & 33.39 $\pm$ 1.08 & 40.97 $\pm$ 0.92 & 49.05 $\pm$ 1.07 & 56.05 $\pm$ 0.95 & 64.48 $\pm$ 0.83  \\
EWC & 49.16 $\pm$ 1.54 & 52.58 $\pm$ 1.85 & 61.0 $\pm$ 1.81 & 67.77 $\pm$ 1.55 & 72.85 $\pm$ 1.18 & 76.8 $\pm$ 1.1 & 80.05 $\pm$ 0.76  \\
AGEM & \textbf{65.48 $\pm$ 1.38} & \textbf{65.93 $\pm$ 1.15} & \textbf{72.95 $\pm$ 0.65} & \textbf{77.23 $\pm$ 0.55} & \textbf{79.38 $\pm$ 0.32} & 81.9 $\pm$ 0.29 & 84.09 $\pm$ 0.25  \\ 
OGD & 44.16 $\pm$ 1.52 & 47.06 $\pm$ 1.26 & 55.75 $\pm$ 0.96 & 63.53 $\pm$ 1.37 & 69.75 $\pm$ 0.36 & 75.67 $\pm$ 0.44 & 80.86 $\pm$ 0.19  \\ 
\algoname & 52.51 $\pm$ 2.3 & 55.81 $\pm$ 1.79 & 65.21 $\pm$ 1.76 & 72.79 $\pm$ 1.34 & 79.0 $\pm$ 0.8 & \textbf{83.34 $\pm$ 0.57} & \textbf{86.65 $\pm$ 0.31}  \\

 \hline \end{tabular}
}

 \bigskip
 
\resizebox{1.0\columnwidth}{!}{
\begin{tabular}{c |  c  c  c  c  c  c  c  c  } 
\hline\hline
 &  \textbf{Task 8}  &  \textbf{Task 9}  &  \textbf{Task 10}  &  \textbf{Task 11}  &  \textbf{Task 12}  &  \textbf{Task 13}  &  \textbf{Task 14}  &  \textbf{Task 15}   \\ \hline
 SGD & 72.75 $\pm$ 0.6 & 78.55 $\pm$ 0.43 & 84.4 $\pm$ 0.26 & 88.45 $\pm$ 0.27 & 91.08 $\pm$ 0.13 & 92.48 $\pm$ 0.09 & 93.28 $\pm$ 0.19 & \textbf{92.74 $\pm$ 0.15}  \\ 
EWC & 82.71 $\pm$ 0.65 & 84.45 $\pm$ 0.37 & 86.32 $\pm$ 0.24 & 87.34 $\pm$ 0.31 & 87.47 $\pm$ 0.37 & 86.9 $\pm$ 0.38 & 85.23 $\pm$ 0.58 & 82.42 $\pm$ 0.65  \\
AGEM & 85.84 $\pm$ 0.12 & 87.47 $\pm$ 0.16 & 89.6 $\pm$ 0.18 & 91.28 $\pm$ 0.09 & 92.54 $\pm$ 0.18 & 93.13 $\pm$ 0.09 & \textbf{93.33 $\pm$ 0.11} & 92.71 $\pm$ 0.12  \\ 
OGD & 84.75 $\pm$ 0.41 & 87.39 $\pm$ 0.38 & 90.09 $\pm$ 0.5 & 91.7 $\pm$ 0.3 & 92.84 $\pm$ 0.28 & 93.03 $\pm$ 0.23 & 92.9 $\pm$ 0.22 & 91.86 $\pm$ 0.19  \\ 
\algoname & \textbf{89.08 $\pm$ 0.1} & \textbf{90.62 $\pm$ 0.2} & \textbf{92.02 $\pm$ 0.3} & \textbf{92.87 $\pm$ 0.09} & \textbf{93.52 $\pm$ 0.17} & \textbf{93.18 $\pm$ 0.12} & 92.85 $\pm$ 0.14 & 91.3 $\pm$ 0.15  \\

 \hline \end{tabular}

}
    \caption{Final Accuracy for Rotated MNIST (averaged over $5$ seeds $\pm 1$ std).}
\end{table}

\begin{table}[h!]
    \centering
\resizebox{1.0\columnwidth}{!}{
\begin{tabular}{c |  c  c  c  c  c  c  c  } 
\hline\hline
 &  \textbf{Task 1}  &  \textbf{Task 2}  &  \textbf{Task 3}  &  \textbf{Task 4}  &  \textbf{Task 5}  &  \textbf{Task 6}  &  \textbf{Task 7}   \\ \hline
SGD & 56.88 $\pm$ 2.85 & 62.18 $\pm$ 6.06 & 65.96 $\pm$ 2.87 & 66.62 $\pm$ 10.57 & 71.74 $\pm$ 3.0 & 70.86 $\pm$ 5.42 & 77.27 $\pm$ 1.69  \\ 
EWC & \textbf{77.5 $\pm$ 2.13} & \textbf{79.78 $\pm$ 1.63} & \textbf{81.91 $\pm$ 0.94} & \textbf{81.92 $\pm$ 0.63} & 81.26 $\pm$ 1.02 & 80.88 $\pm$ 0.89 & 80.91 $\pm$ 1.1  \\ 
AGEM & 75.34 $\pm$ 1.92 & 75.16 $\pm$ 1.37 & 79.41 $\pm$ 0.9 & 79.78 $\pm$ 3.22 & 79.87 $\pm$ 1.65 & 81.16 $\pm$ 1.88 & 82.48 $\pm$ 1.28  \\ 
OGD & 45.97 $\pm$ 3.6 & 63.25 $\pm$ 3.43 & 74.25 $\pm$ 2.8 & 78.9 $\pm$ 3.15 & 80.9 $\pm$ 1.42 & 81.99 $\pm$ 1.27 & 83.86 $\pm$ 0.61  \\ 

\algoname & 35.47 $\pm$ 3.34 & 64.23 $\pm$ 2.05 & 77.98 $\pm$ 1.59 & 80.82 $\pm$ 1.98 & \textbf{83.21 $\pm$ 1.09} & \textbf{84.25 $\pm$ 1.39} & \textbf{85.89 $\pm$ 0.45}  \\ 
 \hline \end{tabular}
 }

 \bigskip
 
\resizebox{1.0\columnwidth}{!}{
\begin{tabular}{c |  c  c  c  c  c  c  c  c  } 
\hline\hline
 &  \textbf{Task 8}  &  \textbf{Task 9}  &  \textbf{Task 10}  &  \textbf{Task 11}  &  \textbf{Task 12}  &  \textbf{Task 13}  &  \textbf{Task 14}  &  \textbf{Task 15}   \\ \hline
SGD & 75.14 $\pm$ 2.34 & 81.54 $\pm$ 2.61 & 83.31 $\pm$ 1.17 & 85.27 $\pm$ 1.71 & 86.48 $\pm$ 0.49 & 88.34 $\pm$ 0.29 & 89.73 $\pm$ 0.48 & \textbf{90.85 $\pm$ 0.16}  \\ 
EWC & 80.98 $\pm$ 0.94 & 80.4 $\pm$ 1.33 & 80.2 $\pm$ 1.18 & 79.77 $\pm$ 1.27 & 78.6 $\pm$ 0.91 & 77.92 $\pm$ 0.8 & 77.3 $\pm$ 0.58 & 76.28 $\pm$ 0.67  \\ 
AGEM & 82.81 $\pm$ 1.09 & 85.86 $\pm$ 0.67 & 85.56 $\pm$ 0.85 & 86.44 $\pm$ 1.22 & 87.65 $\pm$ 0.81 & 88.81 $\pm$ 0.34 & 89.91 $\pm$ 0.25 & 90.79 $\pm$ 0.21  \\ 

OGD & 85.42 $\pm$ 0.59 & 86.69 $\pm$ 0.5 & 86.84 $\pm$ 0.62 & 87.86 $\pm$ 0.65 & 88.68 $\pm$ 0.36 & 89.28 $\pm$ 0.31 & 89.88 $\pm$ 0.23 & 90.45 $\pm$ 0.16  \\ 
\algoname & \textbf{87.08 $\pm$ 0.26} & \textbf{87.46 $\pm$ 0.77} & \textbf{87.9 $\pm$ 0.5} & \textbf{88.34 $\pm$ 0.44} & \textbf{89.16 $\pm$ 0.39} & \textbf{89.65 $\pm$ 0.14} & \textbf{89.93 $\pm$ 0.17} & 90.29 $\pm$ 0.2  \\ 
 \hline \end{tabular}
}
    \caption{Final Accuracy for Permuted MNIST  (averaged over $5$ seeds $\pm 1$ std).}
\end{table}

\begin{table}[h!]
    \centering
\resizebox{1.0\columnwidth}{!}{
\begin{tabular}{c |  c  c  c  c  c  c  c  c  c  c  } 
\hline\hline
 &  \textbf{Task 1}  &  \textbf{Task 2}  &  \textbf{Task 3}  &  \textbf{Task 4}  &  \textbf{Task 5}  &  \textbf{Task 6}  &  \textbf{Task 7}  &  \textbf{Task 8}  &  \textbf{Task 9}  &  \textbf{Task 10}   \\ \hline

SGD & 48.92 $\pm$ 5.11 & 40.96 $\pm$ 2.7 & 46.52 $\pm$ 6.45 & 40.2 $\pm$ 7.27 & 56.36 $\pm$ 8.12 & 44.76 $\pm$ 6.08 & 54.32 $\pm$ 4.9 & 44.52 $\pm$ 5.41 & 46.24 $\pm$ 6.29 & 59.88 $\pm$ 7.56  \\
EWC & \textbf{63.28 $\pm$ 2.44} & 62.32 $\pm$ 7.5 & 57.36 $\pm$ 5.88 & 56.0 $\pm$ 6.39 & 73.56 $\pm$ 6.42 & 52.32 $\pm$ 5.96 & 64.92 $\pm$ 3.03 & 62.44 $\pm$ 2.89 & 53.04 $\pm$ 6.24 & 70.48 $\pm$ 5.09  \\ 
AGEM & 38.88 $\pm$ 2.81 & 37.0 $\pm$ 5.34 & 37.28 $\pm$ 5.62 & 32.44 $\pm$ 8.53 & 42.4 $\pm$ 12.14 & 36.72 $\pm$ 4.66 & 41.12 $\pm$ 5.16 & 39.36 $\pm$ 4.22 & 41.92 $\pm$ 6.3 & 47.08 $\pm$ 8.52  \\ 
OGD & 52.04 $\pm$ 2.92 & 57.84 $\pm$ 5.83 & 59.96 $\pm$ 3.77 & 56.32 $\pm$ 1.47 & 74.12 $\pm$ 2.25 & 58.04 $\pm$ 2.75 & 69.24 $\pm$ 1.18 & 66.36 $\pm$ 3.66 & 60.84 $\pm$ 2.17 & 77.84 $\pm$ 2.2  \\ 
\algoname & 57.24 $\pm$ 2.55 & \textbf{63.08 $\pm$ 4.96} & \textbf{66.16 $\pm$ 0.83} & \textbf{61.52 $\pm$ 1.09} & \textbf{75.32 $\pm$ 6.88} & \textbf{62.88 $\pm$ 1.45} & \textbf{73.28 $\pm$ 1.06} & \textbf{70.48 $\pm$ 1.67} & \textbf{66.32 $\pm$ 1.14} & \textbf{80.28 $\pm$ 1.22}  \\

 \hline \end{tabular}
 }
 
 \bigskip

\resizebox{1.0\columnwidth}{!}{
\begin{tabular}{c |  c  c  c  c  c  c  c  c  c  c  } 
\hline\hline
 &  \textbf{Task 11}  &  \textbf{Task 12}  &  \textbf{Task 13}  &  \textbf{Task 14}  &  \textbf{Task 15}  &  \textbf{Task 16}  &  \textbf{Task 17}  &  \textbf{Task 18}  &  \textbf{Task 19}  &  \textbf{Task 20}   \\ \hline

SGD & 67.56 $\pm$ 4.21 & 49.64 $\pm$ 8.46 & 66.4 $\pm$ 4.6 & 58.68 $\pm$ 4.78 & 65.68 $\pm$ 5.11 & 54.8 $\pm$ 13.41 & 63.6 $\pm$ 3.84 & 57.84 $\pm$ 4.5 & 75.16 $\pm$ 2.54 & 80.12 $\pm$ 2.41  \\ 
EWC & 77.92 $\pm$ 2.36 & 67.4 $\pm$ 1.51 & 74.44 $\pm$ 1.83 & 68.28 $\pm$ 3.08 & 72.64 $\pm$ 1.87 & 67.88 $\pm$ 3.73 & 67.08 $\pm$ 3.67 & 53.08 $\pm$ 5.02 & 73.32 $\pm$ 1.18 & 71.32 $\pm$ 6.05  \\ 
AGEM & 54.68 $\pm$ 7.95 & 39.48 $\pm$ 11.16 & 52.52 $\pm$ 12.46 & 50.36 $\pm$ 7.41 & 55.4 $\pm$ 12.14 & 55.16 $\pm$ 1.8 & 54.92 $\pm$ 9.64 & 50.28 $\pm$ 10.96 & 65.56 $\pm$ 5.36 & 78.52 $\pm$ 1.01  \\ 
OGD & 80.44 $\pm$ 1.93 & 67.8 $\pm$ 2.99 & 80.4 $\pm$ 1.04 & 74.56 $\pm$ 0.8 & 77.92 $\pm$ 1.93 & 72.36 $\pm$ 0.82 & 75.0 $\pm$ 0.83 & \textbf{73.0 $\pm$ 1.48} & 79.52 $\pm$ 1.39 & \textbf{81.88 $\pm$ 1.73}  \\ 
\algoname & \textbf{82.56 $\pm$ 0.74} & \textbf{71.84 $\pm$ 1.81} & \textbf{81.88 $\pm$ 1.52} & \textbf{76.08 $\pm$ 1.51} & \textbf{80.4 $\pm$ 0.95} & \textbf{73.52 $\pm$ 1.39} & \textbf{76.52 $\pm$ 0.53} & \textbf{73.0 $\pm$ 1.63} & \textbf{80.36 $\pm$ 1.82} & 81.28 $\pm$ 0.72  \\ 
 \hline \end{tabular}
}
    \caption{Final Accuracy for Split CIFAR (averaged over $5$ seeds $\pm 1$ std).}
\end{table}

\begin{table}[h!]
    \centering
\resizebox{1.0\columnwidth}{!}{
\begin{tabular}{c |  c  c  c  c  c  } 
\hline\hline
 &  \textbf{Task 1}  &  \textbf{Task 2}  &  \textbf{Task 3}  &  \textbf{Task 4}  &  \textbf{Task 5}   \\ \hline

SGD & 99.35 $\pm$ 0.2 & 88.62 $\pm$ 5.21 & 94.85 $\pm$ 1.69 & 98.1 $\pm$ 0.38 & \textbf{94.6 $\pm$ 0.4}  \\ 
EWC & 99.34 $\pm$ 0.19 & 88.36 $\pm$ 5.38 & 94.96 $\pm$ 1.37 & 98.09 $\pm$ 0.39 & 94.56 $\pm$ 0.49  \\ 
AGEM & 99.5 $\pm$ 0.22 & 85.92 $\pm$ 7.36 & 93.31 $\pm$ 2.13 & 98.07 $\pm$ 0.42 & 94.44 $\pm$ 0.82  \\ 
OGD & 99.64 $\pm$ 0.09 & \textbf{92.24 $\pm$ 1.49} & \textbf{95.75 $\pm$ 0.4} & 98.22 $\pm$ 0.39 & 94.41 $\pm$ 0.4  \\ 
\algoname & \textbf{99.67 $\pm$ 0.08} & 92.22 $\pm$ 1.67 & 95.37 $\pm$ 0.72 & \textbf{98.39 $\pm$ 0.28} & 94.14 $\pm$ 0.42  \\ 
 \hline \end{tabular}
 }
    \caption{Final Accuracy for Split MNIST (averaged over $5$ seeds $\pm 1$ std).}
\end{table}

\newpage

\clearpage

\subsection{Worst-case scenario for \algoname : data spread uniformly along all directions}
\label{toy:worst_case_pca}

In this section, we present a toy example which highlights the drawbacks of \algoname against Catastrophic
Forgetting in comparison with OGD, in the special case where magnitude of eigenvalues are spread out.

% \paragraph[]{Intuition}
% We design the following two edge cases :
% \begin{itemize}
%     \item No-forgetting case :
%     We generate data for a stream of tasks such that the principal components of each task are orthogonal to the
%     principal components to all the other tasks
%     \item Worst forgetting case :
%     We generate data for a stream of tasks such that the eigen values of each task are uniform, and that the
%     principal components are the same for all the tasks.
%     The uniformity of the eigen values implies that no OGD-PCA loses its advantage over a random sampling of the
%     components
%     The shared principal components across all tasks imply that the NTK overlap matrix maximises the
%     Catastrophic Forgetting.
% \end{itemize}

% \paragraph[]{More formally}

% \begin{align*}
%     \intertext{We assume the underlying model is :}
%     y &= \vx^{\top}\w^{\star} \\
%     \intertext{Let $M$ the memory size per task of OGD and OGD-PCA, and let $\Sigma = \sigma \vI$}
%     \intertext{We generate the data $\{X_\tau, \tau \in [T] \}$ such that :}
%     \forall \vx \in \cD_{\tau} \quad \vx &= \Sigma F_{\tau} + \epsilon, \quad \epsilon \sim \cN(0, \vI)
%     \intertext{For the non-forgetting case $\Sigma F_{\tau}$ equals 0 everywhere except on the diagonal terms in $[M \cdot
%     \tau, (M+1) \cdot \tau]$,}
%     \intertext{For the non-forgetting case $\Sigma F_{\tau}$ equals 0 everywhere except on the diagonal terms in $[0,
%     (M+1) \cdot \tau]$.}
% \end{align*}

\paragraph[]{Experiments}

In this section, we build a worst case scneario where datapoints $\{X^{\tau}\}_{\tau=1}^{T}$ are spread uniformly across all directions. We consider a regression task with a linear model $f_{\tau}(X^{\tau})=(X^{\tau})^{\top}({\w_{\tau}(t)-\w_{\tau-1}^{\star}})$ where $X^{\tau} \in \mathbb{R}^{n_{\tau} \times p}$, $\w \in \mathbb{R}^{p}$, $\tau \in [T]$. We generate the data as follows for all $\tau \in [T]$:
\begin{align*}
X^{\tau} \sim \cN(\mu_{x_{\tau},\sigma_{x_{\tau}}}) \\
\w_{\tau}^{\star} \sim \cN(\mu_{\w_{\tau}},\sigma_{\w_{\tau}}) \\
y_{\tau}=(X^{\tau})^{\top}\w_{\tau}^{\star}+\epsilon_{\tau} \\
\epsilon_{\tau} \sim \cN(0,\sigma_{\epsilon_{\tau}})
\end{align*}
% $X^{\tau} \sim \cN(\mu_{x_{\tau},\sigma_{x_{\tau}}})$, $\w_{\tau}^{\star} \sim \cN(\mu_{\w_{\tau}},\sigma_{\w_{\tau}})$ and $y^{\tau}$ such that $y^{\tau}=(X^{\tau})^{\tau}\w_{\tau}^{\star}+\epsilon^{\tau}$ where $\epsilon^{\tau} \sim \cN(0,\sigma_{\epsilon_{\tau}})$,
We are considering Mean Square Error (MSE) for the loss function:
$ \cL_{\tau}=\frac{1}{n_{\tau}}\displaystyle{\sum_{i=1}^{n_{\tau}}}(y^{\tau}_{i}-f_{\tau}(x^{\tau}_{i}))^{2},\hspace{1em}\forall \tau \in [T]$. 

Note in this setting, the kernel is simply the which is simply the gradient kernel matrix of the dataset :
\begin{align*}
\phi(X^{\tau})\phi(X^{\tau})^{\top}=\nabla_{\w}f_{\tau}(X^{\tau})\nabla_{\w}f_{\tau}(X^{\tau})^{\top}=X^{\tau}(X^{\tau})^{\top}  \in \mathbb{R}^{n_{\tau} \times n_{\tau}}   
\end{align*} 
As shown below in Figure \ref{fig:eigenvalues_worst_case} the eigenvalues of the PCA decomposition of $X^{\tau}(X^{\tau})^{\top}$ are of the same magnitude order and taking the first $25$ components only represents $26 \%$ of the explained variance. We trained the model on $15$ tasks with a total memory of $25$ per tasks. We only show below the testing error and forgetting error of the first $9$ tasks. 
As expected, PCA-OGD incurs drastic variation of its loss function while OGD shows practically no forgetting.

\begin{figure}[h!]
\hspace{+0mm}
    \subfloat[
    Eigenvalues structure of the dataset. The first \\
    eigenvalues are sensitively of the same magnitude  \\
    order (left)
    such that taking the first $25 (5\%)$   \\ only explains $26 \%$
    of the data variance (right).]{%
      \includegraphics[width=0.50\textwidth]{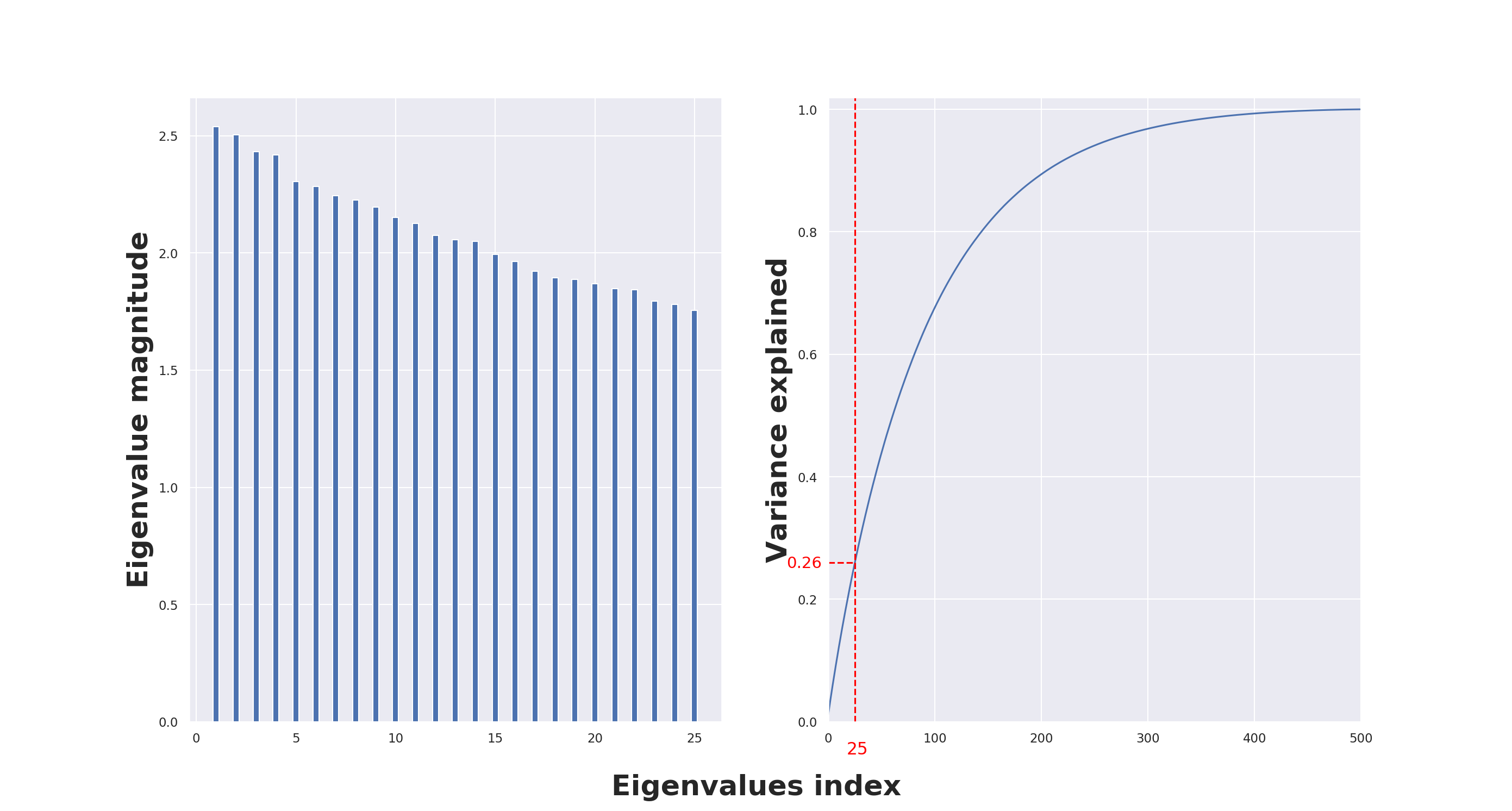}
      \label{fig:eigenvalues_worst_case}
    }
    \hspace{+0mm}
    \subfloat[Testing loss of OGD (left) versus  \\ PCA-OGD (right). 
    OGD incurs almost no forgetting  \\ while PCA-OGD has 
    drastic variation  
    in the testing \\ loss over the time.]{%
      \includegraphics[width=0.50\textwidth]{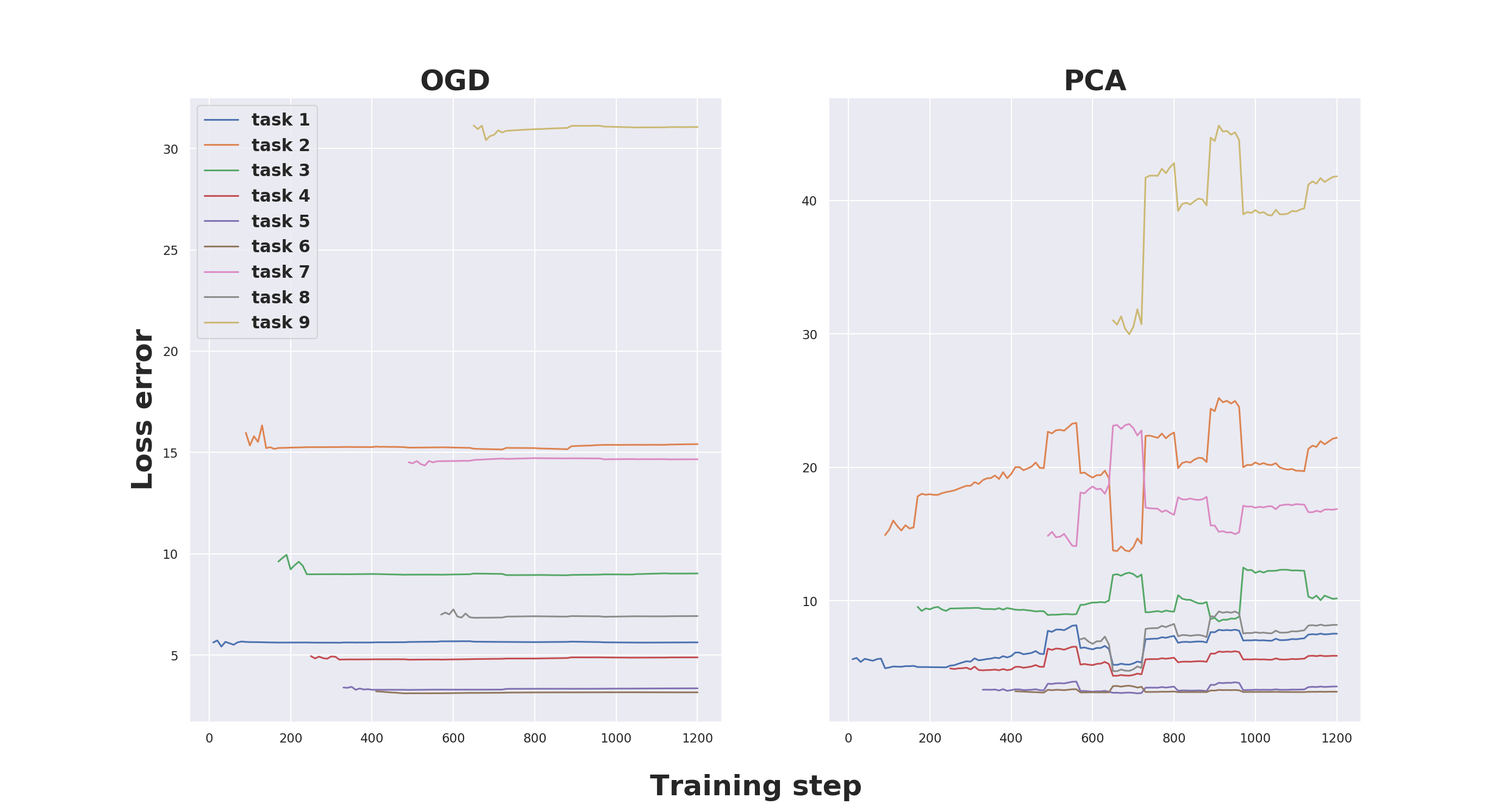}
    }
\end{figure}
\newpage
\subsection{Pseudo-code for GEM-NT}

\begin{algorithm}[h!]
    \SetKwInOut{Input}{Input}
    \SetKwInOut{Output}{Output}
    \SetKw{KwBy}{by}
    \SetAlgoLined
    \Input{A task sequence $\cT_1, \cT_2, \ldots  $, learning rate $\eta$,  components to keep $d$}
%    \Output{The optimal parameter $\vw^\star$}
    \begin{enumerate}
        \item Initialize $S_{1}\leftarrow \{\}$ ; $\w \leftarrow \w_0$
        \item \For{Task ID $\tau=1,2,3, \ldots $}{
        \Repeat{convergence}{
        
        $\vg \leftarrow$ Stochastic Batch Gradient for $\cT_\tau$ at $\w$\;
        \tcp{Orthogonal updates}
       $\tilde{\vg} = \vg - \sum_{(x_{k},y_k) \in \cS_k, {\color{red}k=1,..,\tau-1}} \nabla \cL^{\tau}_{\lambda}(x_{k};y_k))$\;
        $\w \leftarrow \w - \eta \tilde{\vg}$
        }
        % \tcp{Graham-Schmidt orthogonalization}
        % \For{$(\vx, y) \in \cD_\tau $ \text{and} $k \in [1, c]$ \text{s.t.} $y_k = 1$}{
        % $\vecu \leftarrow \nabla f_\tau(\vx; \w) - \sum_{\vv \in \cS_J} \mathrm{proj}_{\vv}(\nabla_{\w} \cL^{\tau}(\vx;
        % \w))$
        % $\cS_J \leftarrow \cS_J \bigcup \{\vecu \} $
        % }
    \tcp{Compute loss gradient}
        {\color{red} Sample $d$ elements $(x_{\tau},y_{\tau})$ from $\cT_{\tau}$}  \
        
        % {\color{red} top $d$ eigenvectors $ \leftarrow PCA (\{ \nabla f_{\tau}(x^{\tau}_{j}) \}_{j=1}^{s} )$} \  
        
        % {\color{red} perform PCA on $\{ \nabla f_{\tau}(x^{\tau}_{j}) \}_{j=1}^{s}$} and $PCA(\cS_{d})$ \
        {\color{red} $S_{\tau}\leftarrow \{ (x_{\tau},y_{\tau}) \}$}  \
        
        }
        % {\color{red} ${\color{red} \cS_J \leftarrow PCA(\cS_{d}) } $}}
    \end{enumerate}
    \caption{GEM-NT for Continual Learning}
    \label{alg:gem-nt}
\end{algorithm}

\subsection{Computational overhead of PCA-OGD}
The only additional step of PCA-OGD over OGD is the \underline{PCA operation}. The complexity of the Graham Schmidt (GS) step is $O(M^{2}p)$ where $M$ is the memory size and $p$ the number of parameters. The PCA operation has complexity  $O(n^{2}p)$  where $n \ll p$ are the first $n$ components to keep. However, $n=M/T$ with $T$ being the number of tasks hence there is no computational overhead for the PCA step: PCA-OGD and OGD have the same asymptotic complexity.

\newpage

\subsection{Eigenvalues evolution of the NTK overlap matrix between the source and target task} \label{app:forgetting_angle}

\begin{figure}[h!]
     \centering
    \includegraphics[width=1.1\linewidth]{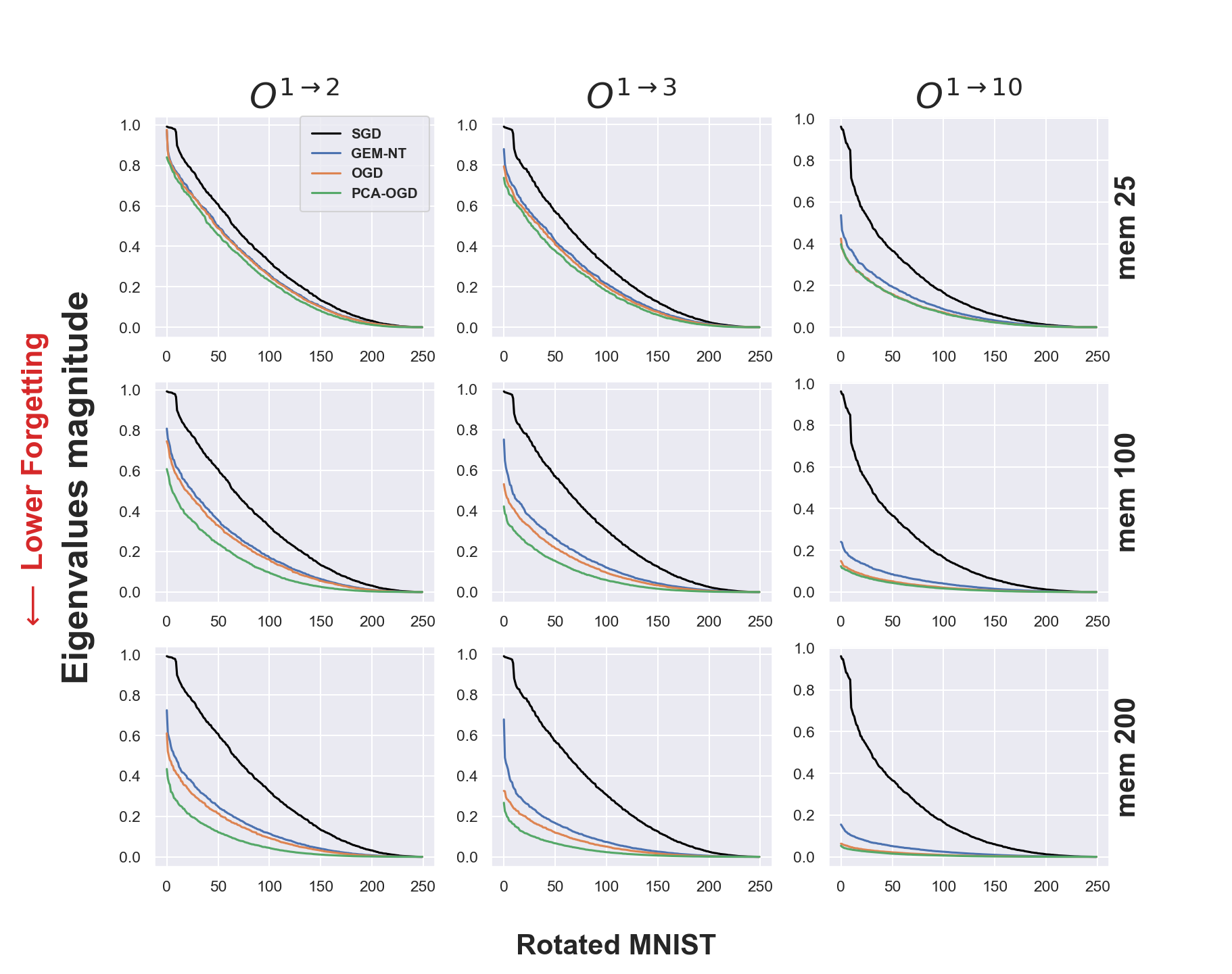}
    \caption{Eigenvalues of the overlap matrix $O^{\tau_S \rightarrow \tau_T}$ for different memozy size and methods. Increasing memory gives better advantage to \algoname.}
\end{figure}

\begin{figure}[h!]
     \centering
    \includegraphics[width=1.1\linewidth]{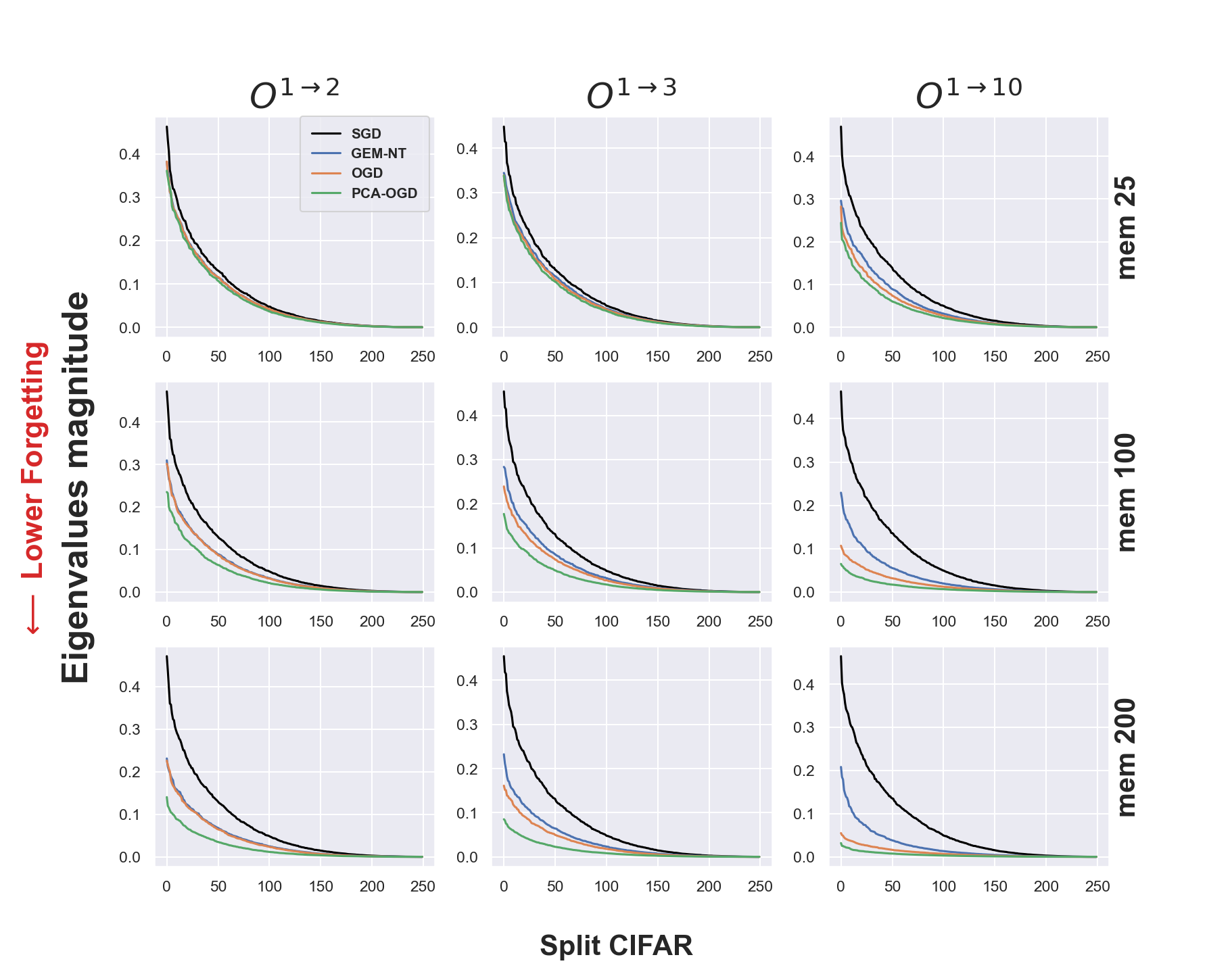} 
    \caption{Eigenvalues of the overlap matrix $O^{\tau_S \rightarrow \tau_T}$ for different memozy size and methods. Increasing memory gives better advantage to \algoname.}
\end{figure}

\begin{figure}[h!]
    \centering
    \includegraphics[width=1.1\linewidth]{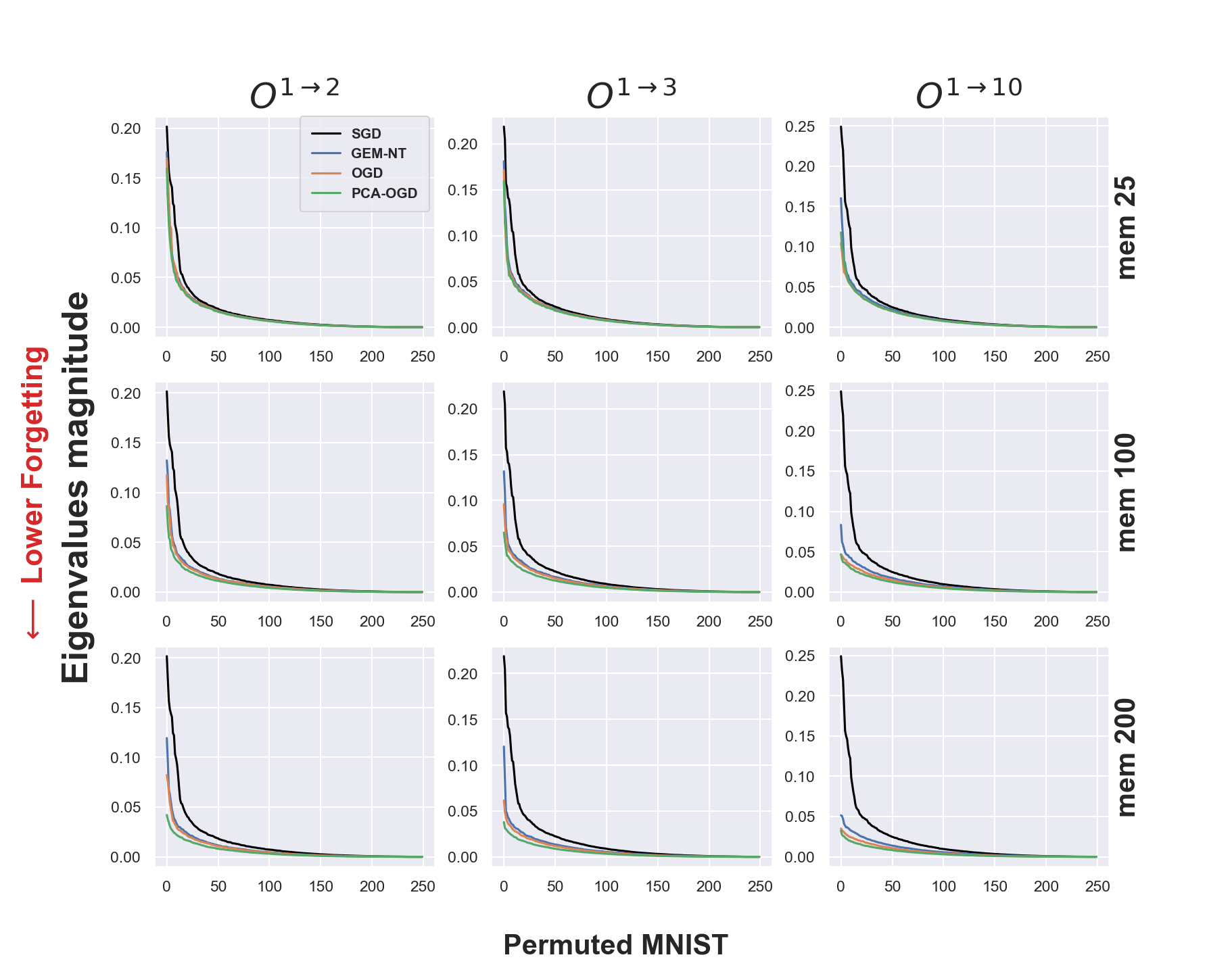}
    \caption{Eigenvalues of the overlap matrix $O^{\tau_S \rightarrow \tau_T}$ for different memozy size and methods. Since there is no Pattern accross task of Permuted MNIST, \algoname does not take advantage of keeping principal eigenvalues directions. }
\end{figure}

\newpage


\clearpage
\newpage
\begin{thebibliography}{36}
\providecommand{\natexlab}[1]{#1}
\providecommand{\url}[1]{\texttt{#1}}
\expandafter\ifx\csname urlstyle\endcsname\relax
  \providecommand{\doi}[1]{doi: #1}\else
  \providecommand{\doi}{doi: \begingroup \urlstyle{rm}\Url}\fi

\bibitem[Aljundi et~al.(2019{\natexlab{a}})Aljundi, Belilovsky, Tuytelaars,
  Charlin, Caccia, Lin, and Page-Caccia]{aljundi2019online}
Rahaf Aljundi, Eugene Belilovsky, Tinne Tuytelaars, Laurent Charlin, Massimo
  Caccia, Min Lin, and Lucas Page-Caccia.
\newblock Online continual learning with maximal interfered retrieval.
\newblock In \emph{Advances in Neural Information Processing Systems}, pages
  11849--11860, 2019{\natexlab{a}}.

\bibitem[Aljundi et~al.(2019{\natexlab{b}})Aljundi, Lin, Goujaud, and
  Bengio]{aljundi2019gradient}
Rahaf Aljundi, Min Lin, Baptiste Goujaud, and Yoshua Bengio.
\newblock Gradient based sample selection for online continual learning.
\newblock In \emph{Advances in Neural Information Processing Systems}, pages
  11816--11825, 2019{\natexlab{b}}.

\bibitem[Arora et~al.(2019)Arora, Du, Hu, Li, Salakhutdinov, and
  Wang]{arora2019exact}
Sanjeev Arora, Simon~S Du, Wei Hu, Zhiyuan Li, Russ~R Salakhutdinov, and
  Ruosong Wang.
\newblock On exact computation with an infinitely wide neural net.
\newblock In \emph{Advances in Neural Information Processing Systems}, pages
  8141--8150, 2019.

\bibitem[Bennani et~al.(2020)Bennani, Doan, and
  Sugiyama]{bennani2020generalisation}
Mehdi~Abbana Bennani, Thang Doan, and Masashi Sugiyama.
\newblock Generalisation guarantees for continual learning with orthogonal
  gradient descent, 2020.

\bibitem[Chaudhry et~al.(2018)Chaudhry, Ranzato, Rohrbach, and Elhoseiny]{agem}
Arslan Chaudhry, Marc'Aurelio Ranzato, Marcus Rohrbach, and Mohamed Elhoseiny.
\newblock Efficient lifelong learning with a-gem.
\newblock \emph{arXiv preprint arXiv:1812.00420}, 2018.

\bibitem[Chen and Liu(2018)]{chen2018lifelong}
Zhiyuan Chen and Bing Liu.
\newblock Lifelong machine learning.
\newblock \emph{Synthesis Lectures on Artificial Intelligence and Machine
  Learning}, 12\penalty0 (3):\penalty0 1--207, 2018.

\bibitem[De~Lange et~al.(2019)De~Lange, Aljundi, Masana, Parisot, Jia,
  Leonardis, Slabaugh, and Tuytelaars]{de2019defying_forgetting}
Matthias De~Lange, Rahaf Aljundi, Marc Masana, Sarah Parisot, Xu~Jia, Ales
  Leonardis, Gregory Slabaugh, and Tinne Tuytelaars.
\newblock Continual learning: A comparative study on how to defy forgetting in
  classification tasks.
\newblock \emph{arXiv preprint arXiv:1909.08383}, 2019.

\bibitem[Du et~al.(2018)Du, Zhai, Poczos, and Singh]{du2018gradient}
Simon~S Du, Xiyu Zhai, Barnabas Poczos, and Aarti Singh.
\newblock Gradient descent provably optimizes over-parameterized neural
  networks.
\newblock \emph{arXiv preprint arXiv:1810.02054}, 2018.

\bibitem[Farajtabar et~al.(2020)Farajtabar, Azizan, Mott, and
  Li]{farajtabar2020orthogonal}
Mehrdad Farajtabar, Navid Azizan, Alex Mott, and Ang Li.
\newblock Orthogonal gradient descent for continual learning.
\newblock In \emph{International Conference on Artificial Intelligence and
  Statistics}, pages 3762--3773, 2020.

\bibitem[Farquhar and Gal(2018)]{farquhar2018towards}
Sebastian Farquhar and Yarin Gal.
\newblock Towards robust evaluations of continual learning.
\newblock \emph{arXiv preprint arXiv:1805.09733}, 2018.

\bibitem[Goodfellow et~al.(2013)Goodfellow, Mirza, Xiao, Courville, and
  Bengio]{goodfellow2013empirical}
Ian~J Goodfellow, Mehdi Mirza, Da~Xiao, Aaron Courville, and Yoshua Bengio.
\newblock An empirical investigation of catastrophic forgetting in
  gradient-based neural networks.
\newblock \emph{arXiv preprint arXiv:1312.6211}, 2013.

\bibitem[Jacot et~al.(2018)Jacot, Gabriel, and Hongler]{jacot2018neural}
Arthur Jacot, Franck Gabriel, and Cl{\'e}ment Hongler.
\newblock Neural tangent kernel: Convergence and generalization in neural
  networks.
\newblock In \emph{Advances in neural information processing systems}, pages
  8571--8580, 2018.

\bibitem[Kemker et~al.(2017)Kemker, McClure, Abitino, Hayes, and
  Kanan]{kemker2017measuring}
Ronald Kemker, Marc McClure, Angelina Abitino, Tyler Hayes, and Christopher
  Kanan.
\newblock Measuring catastrophic forgetting in neural networks.
\newblock \emph{arXiv preprint arXiv:1708.02072}, 2017.

\bibitem[Kirkpatrick et~al.(2017)Kirkpatrick, Pascanu, Rabinowitz, Veness,
  Desjardins, Rusu, Milan, Quan, Ramalho, Grabska-Barwinska, et~al.]{ewc}
James Kirkpatrick, Razvan Pascanu, Neil Rabinowitz, Joel Veness, Guillaume
  Desjardins, Andrei~A Rusu, Kieran Milan, John Quan, Tiago Ramalho, Agnieszka
  Grabska-Barwinska, et~al.
\newblock Overcoming catastrophic forgetting in neural networks.
\newblock \emph{Proceedings of the national academy of sciences}, 114\penalty0
  (13):\penalty0 3521--3526, 2017.

\bibitem[Krizhevsky et~al.(2009)]{krizhevsky2009learning}
Alex Krizhevsky et~al.
\newblock Learning multiple layers of features from tiny images.
\newblock 2009.

\bibitem[LeCun et~al.(1998)LeCun, Bottou, Bengio, and
  Haffner]{lecun1998gradient}
Yann LeCun, L{\'e}on Bottou, Yoshua Bengio, and Patrick Haffner.
\newblock Gradient-based learning applied to document recognition.
\newblock \emph{Proceedings of the IEEE}, 86\penalty0 (11):\penalty0
  2278--2324, 1998.

\bibitem[Lee et~al.(2019)Lee, Xiao, Schoenholz, Bahri, Novak, Sohl-Dickstein,
  and Pennington]{lee2019wide}
Jaehoon Lee, Lechao Xiao, Samuel~S Schoenholz, Yasaman Bahri, Roman Novak,
  Jascha Sohl-Dickstein, and Jeffrey Pennington.
\newblock Wide neural networks of any depth evolve as linear models under
  gradient descent.
\newblock \emph{arXiv preprint arXiv:1902.06720}, 2019.

\bibitem[Lopez-Paz and Ranzato(2017)]{gem}
David Lopez-Paz and Marc'Aurelio Ranzato.
\newblock Gradient episodic memory for continual learning.
\newblock In \emph{Advances in neural information processing systems}, pages
  6467--6476, 2017.

\bibitem[Mallya and Lazebnik(2018)]{mallya2018packnet}
Arun Mallya and Svetlana Lazebnik.
\newblock Packnet: Adding multiple tasks to a single network by iterative
  pruning.
\newblock In \emph{Proceedings of the IEEE Conference on Computer Vision and
  Pattern Recognition}, pages 7765--7773, 2018.

\bibitem[McCloskey and Cohen(1989)]{mccloskey1989catastrophic}
Michael McCloskey and Neal~J Cohen.
\newblock Catastrophic interference in connectionist networks: The sequential
  learning problem.
\newblock In \emph{Psychology of learning and motivation}, volume~24, pages
  109--165. Elsevier, 1989.

\bibitem[Mirzadeh et~al.(2020)Mirzadeh, Farajtabar, Pascanu, and
  Ghasemzadeh]{mirzadeh2020understanding}
Seyed~Iman Mirzadeh, Mehrdad Farajtabar, Razvan Pascanu, and Hassan
  Ghasemzadeh.
\newblock Understanding the role of training regimes in continual learning.
\newblock \emph{arXiv preprint arXiv:2006.06958}, 2020.

\bibitem[Nguyen et~al.(2019)Nguyen, Achille, Lam, Hassner, Mahadevan, and
  Soatto]{nguyen2019forgetting}
Cuong~V Nguyen, Alessandro Achille, Michael Lam, Tal Hassner, Vijay Mahadevan,
  and Stefano Soatto.
\newblock Toward understanding catastrophic forgetting in continual learning.
\newblock \emph{arXiv preprint arXiv:1908.01091}, 2019.

\bibitem[Nguyen et~al.(2020)Nguyen, Chen, Jun, and Kim]{nguyen2020explaining}
Giang Nguyen, Shuan Chen, Tae~Joon Jun, and Daeyoung Kim.
\newblock Explaining how deep neural networks forget by deep visualization,
  2020.

\bibitem[Pan et~al.(2020)Pan, Swaroop, Immer, Eschenhagen, Turner, and
  Khan]{pan2020continual}
Pingbo Pan, Siddharth Swaroop, Alexander Immer, Runa Eschenhagen, Richard~E
  Turner, and Mohammad~Emtiyaz Khan.
\newblock Continual deep learning by functional regularisation of memorable
  past.
\newblock \emph{arXiv preprint arXiv:2004.14070}, 2020.

\bibitem[Parisi et~al.(2019)Parisi, Kemker, Part, Kanan, and
  Wermter]{parisi2019icarl}
German~I Parisi, Ronald Kemker, Jose~L Part, Christopher Kanan, and Stefan
  Wermter.
\newblock Continual lifelong learning with neural networks: A review.
\newblock \emph{Neural Networks}, 113:\penalty0 54--71, 2019.

\bibitem[Ramasesh et~al.(2020)Ramasesh, Dyer, and Raghu]{ramasesh2020anatomy}
Vinay~V Ramasesh, Ethan Dyer, and Maithra Raghu.
\newblock Anatomy of catastrophic forgetting: Hidden representations and task
  semantics.
\newblock \emph{arXiv preprint arXiv:2007.07400}, 2020.

\bibitem[Rosenfeld and Tsotsos(2018)]{rosenfeld2018incremental}
Amir Rosenfeld and John~K Tsotsos.
\newblock Incremental learning through deep adaptation.
\newblock \emph{IEEE transactions on pattern analysis and machine
  intelligence}, 2018.

\bibitem[Silver(2011)]{agi}
Daniel~L Silver.
\newblock Machine lifelong learning: challenges and benefits for artificial
  general intelligence.
\newblock In \emph{International conference on artificial general
  intelligence}, pages 370--375. Springer, 2011.

\bibitem[Thrun(1995)]{thrun1995lifelong}
Sebastian Thrun.
\newblock A lifelong learning perspective for mobile robot control.
\newblock In \emph{Intelligent robots and systems}, pages 201--214. Elsevier,
  1995.

\bibitem[Toneva et~al.(2018)Toneva, Sordoni, Combes, Trischler, Bengio, and
  Gordon]{toneva2018empirical}
Mariya Toneva, Alessandro Sordoni, Remi Tachet~des Combes, Adam Trischler,
  Yoshua Bengio, and Geoffrey~J Gordon.
\newblock An empirical study of example forgetting during deep neural network
  learning.
\newblock \emph{arXiv preprint arXiv:1812.05159}, 2018.

\bibitem[Wedin(1983)]{wedin1983angles}
Per~{\AA}ke Wedin.
\newblock On angles between subspaces of a finite dimensional inner product
  space.
\newblock In \emph{Matrix Pencils}, pages 263--285. Springer, 1983.

\bibitem[Xie et~al.(2020)Xie, He, Fu, Sato, Tao, and
  Sugiyama]{xie2020artificial}
Zeke Xie, Fengxiang He, Shaopeng Fu, Issei Sato, Dacheng Tao, and Masashi
  Sugiyama.
\newblock Artificial neural variability for deep learning: On overfitting,
  noise memorization, and catastrophic forgetting.
\newblock \emph{arXiv preprint arXiv:2011.06220}, 2020.

\bibitem[Yin et~al.(2020)Yin, Farajtabar, and Li]{yin2020sola}
Dong Yin, Mehrdad Farajtabar, and Ang Li.
\newblock Sola: Continual learning with second-order loss approximation.
\newblock \emph{arXiv preprint arXiv:2006.10974}, 2020.

\bibitem[Yu et~al.(2020)Yu, Kumar, Gupta, Levine, Hausman, and
  Finn]{yu2020gradient}
Tianhe Yu, Saurabh Kumar, Abhishek Gupta, Sergey Levine, Karol Hausman, and
  Chelsea Finn.
\newblock Gradient surgery for multi-task learning.
\newblock \emph{Advances in Neural Information Processing Systems}, 33, 2020.

\bibitem[Zenke et~al.(2017)Zenke, Poole, and Ganguli]{zenke2017continual}
Friedemann Zenke, Ben Poole, and Surya Ganguli.
\newblock Continual learning through synaptic intelligence.
\newblock \emph{Proceedings of machine learning research}, 70:\penalty0 3987,
  2017.

\bibitem[Zhu and Knyazev(2013)]{zhu2013angles}
Peizhen Zhu and Andrew~V Knyazev.
\newblock Angles between subspaces and their tangents.
\newblock \emph{Journal of Numerical Mathematics}, 21\penalty0 (4):\penalty0
  325--340, 2013.

\end{thebibliography}
\end{document}